\documentclass{article}
\usepackage[preprint]{neurips_2026}

\usepackage{amsmath,amsfonts,bm}

\def\eqref#1{equation~\ref{#1}}

\def\1{\bm{1}}

\DeclareMathAlphabet{\mathsfit}{\encodingdefault}{\sfdefault}{m}{sl}
\SetMathAlphabet{\mathsfit}{bold}{\encodingdefault}{\sfdefault}{bx}{n}

\newcommand{\E}{\mathbb{E}}

\newcommand{\R}{\mathbb{R}}

\newcommand{\KL}{D_{\mathrm{KL}}}

\usepackage{booktabs}
\usepackage{siunitx}

\usepackage{mwe}
\usepackage{wrapfig}
\usepackage{bbm}
\usepackage[english]{babel}
\usepackage{amsmath,amsfonts,amssymb, amsthm}
\usepackage{graphicx}
\usepackage[colorlinks=true, allcolors=blue]{hyperref}
\usepackage{parskip}
\usepackage{tabularx}
\usepackage{bm}
\usepackage{hhline}
\usepackage{multirow}
\usepackage{algorithm}
\usepackage{etoolbox}
\usepackage{subcaption}
\usepackage[dvipsnames]{xcolor}
\usepackage{mathtools}
\usepackage{nth}
\usepackage{diagbox}
\usepackage{float}
\usepackage{tabstackengine}
\usepackage{tablefootnote}
\usepackage{soul}
\usepackage{caption}
\usepackage{algorithm}
\usepackage{algpseudocode}

\newtheorem{theorem}{Theorem}
\newtheorem{proposition}{Proposition}
\newtheorem{corollary}{Corollary}

\renewcommand{\eqref}[1]{(\ref{#1})}

\newcommand{\N}{\mathcal{N}}
\renewcommand{\KL}{\mathrm{D}_{\mathrm{KL}}}

\renewcommand{\R}{\mathbb{R}}
\renewcommand{\E}{\mathbb{E}}

\renewcommand{\d}{\mathrm{d}}

\begin{document}

\title{Stochastic Transition-Map Distillation for\\ Fast Probabilistic Inference}

\author{%
  George Rapakoulias $^1$
  \And
  Peter Garud $^{2}$
  \And
  Lingjiong Zhu $^3$ 
  \And
  Panagiotis Tsiotras $^1$ 
  \\
  \\
  $^1$ Department of Aerospace Engineering, Georgia Institute of Technology, Atlanta, GA \\
  $^2$ Department of Computer Science, Georgia Institute of Technology, Atlanta, GA \\
  $^3$ Department of Mathematics, Florida State University, Tallahassee, FL
}

\maketitle
\begingroup
\renewcommand{\thefootnote}{}
\footnotetext{Corresponding authors: \texttt{grap@gatech.edu, lzhu2@fsu.edu, tsiotras@gatech.edu}}
\addtocounter{footnote}{-1}
\endgroup

\begin{abstract}
Diffusion models achieve strong generation quality, diversity, and distribution coverage, but their performance often comes with expensive inference.
In this work, we propose \textit{Stochastic Transition-Map Distillation} (STMD), a teacher-free framework for accelerating diffusion model inference while preserving probabilistic sample generation.
In contrast to score-based diffusion models, whose denoising parametrization models the mean of the posterior distribution, STMD distills the full transition map associated with the sampling stochastic differential equation (SDE).
We parameterize these SDE transitions with a conditional Mean Flow model, yielding a one- or few-step stochastic sampler that retains the transition structure of the underlying diffusion process.
This perspective is especially useful for downstream tasks that require stochastic inference, such as diffusion posterior sampling, inverse problems, and energy-based fine-tuning.
Compared to recent distillation methods, STMD requires no pretrained teacher, bi-level optimization, or trajectory simulation and caching, enabling efficient and scalable training.
We derive convergence bounds for our method in the Wasserstein distance, providing a strong theoretical foundation for our approach, and validate STMD on various image generation examples on the MNIST, CIFAR-10, and CelebA datasets.
\end{abstract}

%
\section{Introduction and Motivation}

Continuous-time diffusion models excel in generation quality and sample diversity compared to their predecessors, such as Generative Adversarial Networks (GANs) and Variational Autoencoders (VAEs) \citep{goodfellow2014GAN,goodfellow2020generative, arjovsky2017wasserstein, kingma2014VAE,kingma2019introduction}; 
however, this performance increase comes at the expense 
of inference complexity, stemming from the need to integrate the underlying continuous-time ordinary differential or stochastic differential equations.
Although many techniques have been proposed to accelerate inference, such as distillation \citep{luhman2021knowledge, song2023consistency} and Optimal Transport (OT) based approaches~\citep{liu2022flow, shi2023diffusion}, to name a few, achieving the optimal tradeoff between generation quality and sample diversity versus inference complexity is still unresolved~\citep{dieleman2024distillation}.

Flow Matching (FM) and its variants \citep{liu2022flow, lipman2022flow, albergo2023building} have emerged as a powerful tool for training continuous-time flow models.
They work by approximating the drift term (also referred to as the policy) of an Ordinary Differential Equation (ODE) or Stochastic Differential Equation (SDE) as a mixture of elementary point-to-point conditional drift terms.
Implementing Flow Matching in practice, however, requires a temporal discretization of the learned continuous-time ODE/SDE with a very small time step size, slowing down inference~\citep{shi2023diffusion, lipman2022flow}.

Distillation techniques offer a set of tools to accelerate inference by learning a flow map of the corresponding ODE directly \citep{song2023consistency, boffi2025build, Boffi2025flow, geng2025mean, Geng-easy-2025, salimans_progressive_2022}.
At large, these methods leverage different techniques to ``learn to integrate'' an underlying continuous-time ODE.
Although they achieve high-quality results, one disadvantage is that they are usually used only with deterministic inference, that is, given an initial noise sample, the generated sample is unique.
Many downstream tasks, however, benefit from stochastic inference to boost generation quality, improve model alignment with user queries \citep{holderrieth_diamond_2026}, or perform diffusion posterior inference based on noisy or partial sample observations \citep{chung_diffusion_2024}. 
In this direction, recent works have explored training one- or few-step generators to match the distribution of a diffusion model \citep[Section 10.2]{lai2025principles} \citep{yin2024one, gushchin2025inverse, peng_noise_2025}. 

In this work, we take a different approach and propose a method for learning the transition probabilities of a given SDE, thereby generalizing ODE-based distillation approaches to the stochastic regime.
We achieve this by using a conditional Mean Flow, which learns a single or few-step generator for the transition probability of a diffusion model. 

The contributions of this paper are as follows: 
\begin{itemize}
    \item We propose STMD, a distillation framework based on conditional Mean Flows, enabling us to learn the transition maps of high-dimensional diffusion SDEs using a single-step or a few-step generator through a regression objective, and without using any pretrained model.

    \item We provide a theoretical convergence analysis in the 2-Wasserstein distance for the Mean Flow algorithm, and then extend it to study the convergence of our learned transition map to that of the underlying diffusion process. To the best of our knowledge, this is the first Mean Flow convergence analysis in the 2-Wasserstein distance.     

    \item We illustrate this method in various scenarios related to image generation and inpainting problems, showing strong empirical performance.
\end{itemize}

\section{Preliminaries}

\subsection{Flow Matching with Straight-Line Interpolants} \label{sec_flow_matching}

We first recall the basic Flow Matching formulation \citep{liu2022flow, lipman2022flow, albergo2023building} with straight-line interpolants.
All distributions are supported in the $d$-dimensional Euclidean space, denoted $\R^d$.
To align with the recently proposed Mean Flow model \citep{geng2025mean}, we will use $z_0 \sim \rho_{0}$ to denote the data distribution, and $z_1 \sim \rho_1$ to denote a sample from a simple prior distribution, e.g., a Gaussian distribution. 
We denote the densities of flow-related variables with $\rho$, and use the variable $s$ to denote time. 
Unless stated otherwise, all differential equations run backward in time, that is, from $s=1$ to $ s=0$.
For $s \in [0,1]$, define the straight-line interpolant 
\begin{equation}
    z_s = (1-s)z_0 + s z_1 .
\end{equation}
The corresponding conditional velocity is constant along the path and is given by
\begin{equation}
    v_s( z_s | z_0,z_1) = z_1-z_0 .
\end{equation}
A velocity field that transports $\rho_0$ to $\rho_1$ and vice versa can be obtained by averaging this conditional velocity over all pairs $(z_0,z_1) \sim \rho_0 \times \rho_1$ that could have produced the same intermediate point $z_s$:
\begin{equation} \label{mixture_velocity}
    v_s(z) = \E \left[ z_1-z_0 \mid z_s=z\right].
\end{equation}
The Flow Matching objective trains a neural network $v^\theta_s(z)$ to approximate \eqref{mixture_velocity} using a finite-sample approximation of the conditional flow matching objective \citep{lipman2022flow, liu2022flow}:
\begin{equation} \label{CFM_objective}
    \mathcal L_{\mathrm{CFM}} = \E \left[ \left\| v^\theta_s(z_s) - (z_1-z_0) \right\|^2 \right].
\end{equation}
This conditional objective is equivalent, in terms of its optimizer, to regression against the marginal velocity $v_s(z)$ defined in \eqref{mixture_velocity}. 
Indeed, by expanding \eqref{CFM_objective} and using the tower property of conditional expectation,
\begin{equation}
    \E \left[ \left\| v^\theta_s(z_s) - (z_1-z_0) \right\|^2 \right] =
    \E \left[ \left\| v^\theta_s(z_s) - v_s(z_s) \right\|^2 \right]+C,
\end{equation}
where $C$ is independent of $\theta$, motivating the use of CFM objective \eqref{CFM_objective}.

At inference time, samples are generated by integrating the learned ODE backwards from $s=1$ to $s=0$:
\begin{equation}\label{learned_flow}
    \dot z_s = v^\theta_s(z_s).
\end{equation}

\subsection{Consistency Models and Mean Flows} \label{sec_meanflow}

\begin{figure}[!ht]
    \centering
    \begin{subfigure}{0.3\linewidth}
        \centering
        \includegraphics[width=\linewidth]{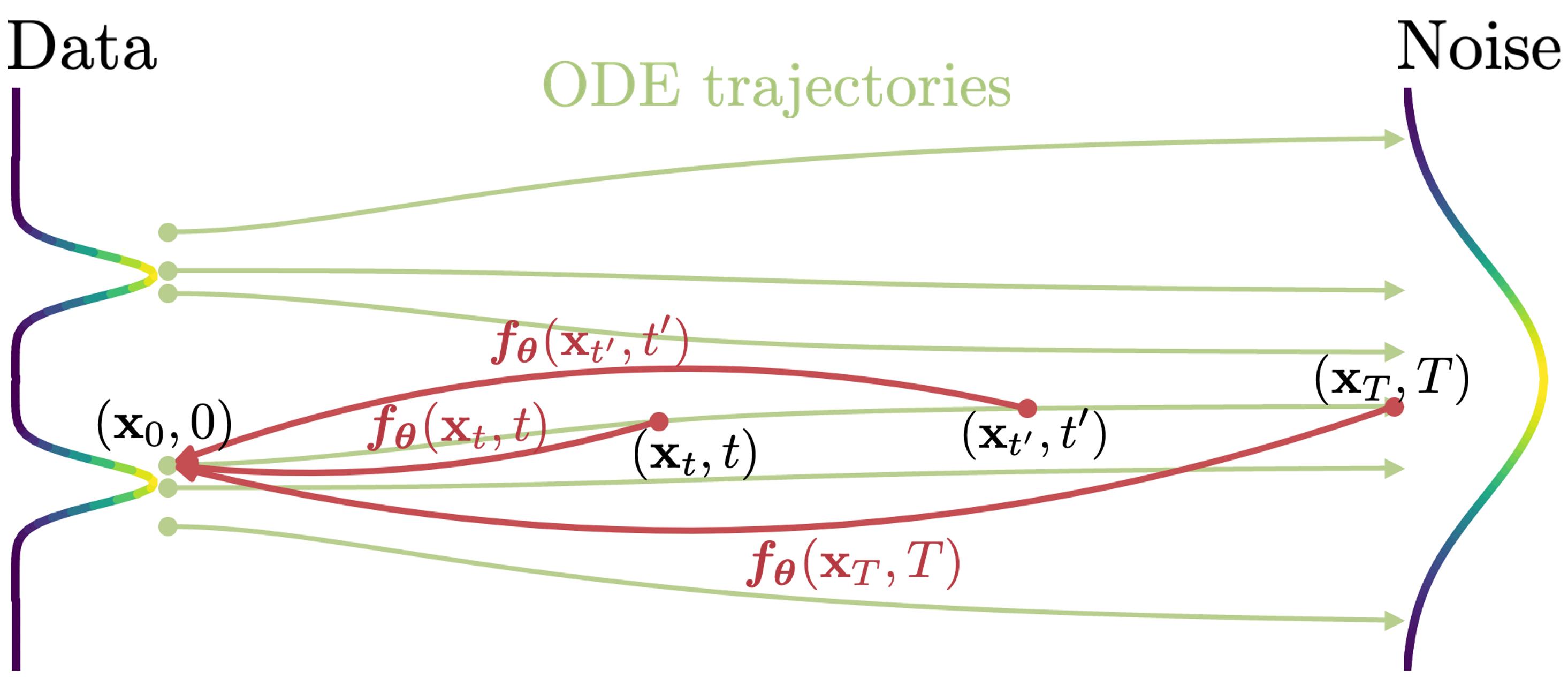}
        \caption{}
        \label{fig:consistency_model}
    \end{subfigure}
    \hfill
    \begin{subfigure}{0.68\linewidth}
        \centering
        \includegraphics[width=\linewidth]{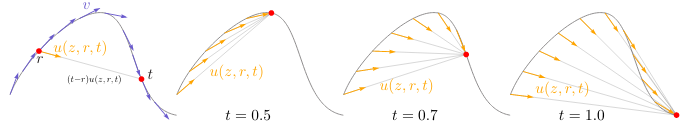}
        \caption{}
        \label{fig:meanflow}
    \end{subfigure}

    \caption{(a): Consistency training learns the flow map of an ODE without requiring sampling of simulated trajectories during training. Figure taken from \cite{song2023consistency}. (b): Mean Flows parametrize the flow map using the average velocity along a trajectory, allowing for taking exact, finite-time steps during inference. Figure taken from \cite{geng2025mean}.}
    \label{fig:cm_meanflow_side_by_side}
\end{figure}

A limitation of flow models is that during inference, the learned flow \eqref{learned_flow} needs to be integrated.
To accelerate inference, distillation approaches aim at learning the flow map of the inference ODE:
\begin{equation} \label{fm_ode}
    \dot z_s = v_s(z_s). 
\end{equation}
Seminal works in this direction include \cite{song2023consistency, kim2023consistency}, which are based on the probability flow ODE to model $v_s$ in \eqref{fm_ode}, and \cite{boffi2025build, geng2025mean}, that leverage a flow-matching ODE and use a drift of the form \eqref{mixture_velocity}.

A recent work is the Mean Flow (MF) model \citep{geng2025mean}, which learns the flow map of \eqref{learned_flow} without requiring a pretrained reference model $v^\theta_s(z_s)$.
Consider an ODE~\eqref{fm_ode} with drift as in~\eqref{mixture_velocity}. 
Let $0 \leq r  < s \leq 1$. 
Consider a trajectory starting from $z_s$ at time $s$, denoted with $\{z_\tau, \, \tau \in [r, s]\}$. 
The MF algorithm is based on the concept of average velocity along the flow lines of \eqref{learned_flow}, i.e., the quantity: 
\begin{equation} \label{mean_flow_property}
    u(z_s, r, s) = \frac{1}{s-r} \int_r^s v_\tau(z_\tau) \, \d \tau. 
\end{equation}
Since, by definition of a solution of the ODE \eqref{learned_flow}, 
\begin{equation}
    z_s =  z_r  -  \int_r^s v_\tau(z_\tau) \, \d \tau = z_r - (s-r) \, u(z_s, r, s),
\end{equation}
so that a model for $u(z_s, r, s)$ fully describes the flow map of \eqref{learned_flow}.  
Rearranging and differentiating \eqref{mean_flow_property} with respect to time $s$ gives 
\begin{equation} \label{mean_flow}
    u(z_s, r, s) = v_s(z_s) - (s-r) \frac{\d }{\d s}u(z_s, r, s),
\end{equation}
while
\begin{equation} \label{total_derivative}
     \frac{\d }{\d s}u(z_s, r, s) = \frac{\partial u}{ \partial z} v_s(z_s)  + \frac{\partial u}{\partial s}. 
\end{equation}

The MF model parametrizes the mean velocity in Equation $\eqref{mean_flow_property}$ with a neural network, and is trained through the MF objective function, which is a square-norm deviation penalty based on \eqref{mean_flow}:
\begin{equation} \label{MF}
    \mathcal{L}_{\mathrm{MF}} = \E\left[ \left\| u^{\theta}(z_s, r, s) -\texttt{sg}\!\left( v_s(z_s) - (s-r)\frac{\d}{\d s} u^{\theta}(z_s, r, s) \right) \right\|^2 \right].
\end{equation}
The \texttt{sg} operator in \eqref{MF} denotes the stop-grad operator commonly used in the literature \citep{song2023consistency,Song-2024,Frans2025,Geng-easy-2025,Lu-2025,geng2025mean}, and is used to stabilize training by avoiding having to backpropagate through the numerical calculation of the Jacobian-vector product related to Equation \eqref{total_derivative}.
Since the expression for $v_{s}(z_s) = \E[z_1 - z_0|z_s] $ is not known a priori, similarly with flow matching, the objective \eqref{MF} is replaced with a tractable conditional objective, and the MF model is trained instead with:
\begin{equation}\label{CMF}
    \mathcal{L}_{\mathrm{CMF}} = \E\left[ \left\| u^{\theta}(z_s, r, s) -\texttt{sg}\!\left( (z_1 - z_0) - (s-r)\frac{\d}{\d s} u^{\theta}(z_s, r, s) \right) \right\|^2 \right].
\end{equation}
The following proposition guarantees the equivalence of minimizing $\mathcal{L}_{\mathrm{MF}}$ in \eqref{MF} and $\mathcal{L}_{\mathrm{CMF}}$ in \eqref{CMF}.

\begin{proposition} \label{MFprop}
    The loss function $\mathcal{L}_{\mathrm{MF}}$  in \eqref{MF} and the loss function   $\mathcal{L}_{\mathrm{CMF}}$ given in \eqref{CMF}, have the same gradient with respect to the model parameters, that is  $\nabla_{\theta} \mathcal{L}_{\mathrm{CMF}} = \nabla_\theta \mathcal{L}_{\mathrm{MF}}$. 
\end{proposition}

\section{Stochastic Transition-Map Distillation}

\subsection{Unconditional Diffusion Distillation} \label{unc_stmd}
\begin{wrapfigure}{r}{0.45 \linewidth}
    \centering
    \includegraphics[width=0.99\linewidth, trim=0.2cm 0 0.2cm 1.5cm, clip ]{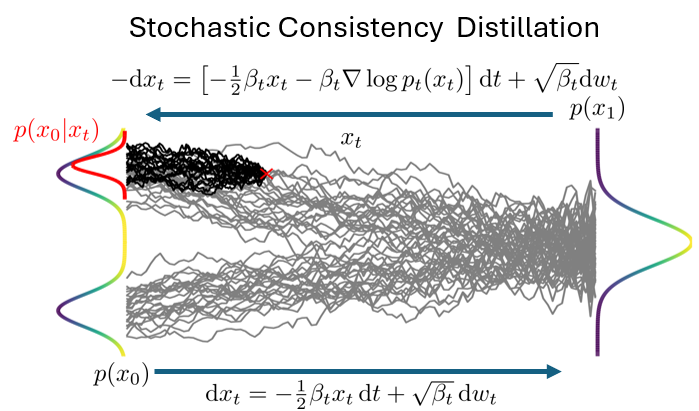}
    \caption{Stochastic Transition-Map Distillation.}
    \label{fig:STMD}
\end{wrapfigure}

Consider a forward diffusion process
\begin{equation} \label{noise_sde}
    \d x_t = - \frac{1}{2} \beta_t x_t \, \d t + \sqrt{\beta_t} \, \d w_t, \quad x_0 \sim p_0,
\end{equation}
where $w_t$ is a standard $d$-dimensional Brownian motion and $\beta_{t}\geq 0$ is a scalar function of time.
To avoid confusion with the Mean Flow variable $z_s$, we will denote probability densities associated with the diffusion variable $x_t$ using $p_t$, and denote the diffusion time with $t$.
The SDE \eqref{noise_sde} is known as the variance-preserving (VP) SDE in the literature \citep{song2020score}.
Since \eqref{noise_sde} is an Ornstein-Uhlenbeck process, the transition density function $p(x_t|x_0)$ of \eqref{noise_sde} can be parametrized in closed form:
\begin{equation}\label{noise_kernel}
    x_t = \alpha_t x_0 + \sigma_t \epsilon, \quad \epsilon \sim \N(0, I_{d}),\quad x_0 \sim p_{0},
\end{equation}
where $\alpha_t^2 := \exp( -\int_0^t \beta_\tau \d \tau)$, $\sigma_t := \sqrt{1-\alpha_t^2}$. 
Furthermore, the time reversal of \eqref{noise_sde}
is an SDE running backward in time \citep{Anderson1982,Follmer1985,song2020score,Cattiaux2023}:
\begin{equation} \label{sampling_SDE}
    -\d x_t = \left[ -\frac{1}{2} \beta_t x_t - \beta_t \nabla \log  p_t(x_t) \right] \d t + \sqrt{ \beta_t} \d w_t, \quad x_1 \sim p_1,
\end{equation}
where $\nabla \log p_t(x_t)$ is known as the score function \citep{Schervish1995}. It is easy to verify that the marginal of \eqref{sampling_SDE} matches that of \eqref{noise_sde} for all time marginals \citep{Anderson1982,song2020score}. 
Score-based diffusion models learn the score function using Tweedie's formula \citep{Efron2011}:
\begin{equation} \label{score_function}
    \nabla \log  p_t(x_t) = \frac{\alpha_t \E[x_0|x_t] - x_t}{\sigma_t^2},
\end{equation}
and by learning the mean of the posterior distribution $p(x_0|x_t)$.
Our goal, instead, is to model the weak solution of the SDE \eqref{sampling_SDE}, that is, given any point $x_t$ along the stochastic trajectories of \eqref{sampling_SDE}, and some time $t'<t$, predict the corresponding distribution for the state $x_t' \sim p(x_{t'}|x_t)$, using a single step generator, without incurring the costly numerical integration of \eqref{sampling_SDE}. We illustrate this concept in Figure \ref{fig:STMD}.

Instead of modeling $p(x_{t'}|x_t)$ directly, we will use an $x_0$-parametrization for our model
\begin{equation} \label{x0_parametrization}
    p(x_{t'}|x_t) = \int_{\mathbb{R}^{d}} p(x_0|x_t) p(x_{t'}|x_0, x_t) \, \d x_0. 
\end{equation}
Since, for all $0\leq t' < t \leq 1$, the bridge measure $p(x_{t'}|x_0, x_t)$ is available in closed form for diffusions of the form \eqref{noise_sde} using~\cite{song2020score}:
\begin{equation} \label{bridge_measure}
p(x_{t'} | x_0,x_t) = \mathcal N \left( x_{t'}; \frac{\alpha_{t'}^2-\alpha_t^2} {\alpha_{t'}\sigma_t^2}x_0 + \frac{\alpha_t\sigma_{t'}^2}{\alpha_{t'}\sigma_t^2}x_t,
\frac{\sigma_{t'}^2(\alpha_{t'}^2-\alpha_t^2)}
{\alpha_{t'}^2\sigma_t^2}I_{d}
\right),
\qquad 0<t'<t,
\end{equation}
given a parametric model $p^\theta(x_0|x_t)$, sampling from $p(x_{t'}|x_t)$ through \eqref{x0_parametrization} is as easy as sampling from $p^\theta(x_0|x_t)$.

We therefore parameterize $p^\theta(x_0|x_t)$ with a conditional Mean Flow, that is, we construct a flow model that maps the prior $z_1 \sim \N(0, I_{d})$ to samples from $z_0 \sim p(x_0|x_t)$. 
Specifically, we define a conditional flow model in the latent variable $z_s$, conditioned on $x_t$ and $t$ as follows:
\begin{equation} \label{cond_meanflow}
    u(z_s, r, s, x_t, t) = v_s(z_s| x_t, t) - (s-r) \frac{\d }{\d s}u(z_s, r, s, x_t, t),
\end{equation}
where $v_s(z_s|x_t, t) = \E[z_1 - x_0|z_s, x_t, t]$. 
We now define the diffusion MF loss function as 
\begin{equation} \label{condMF}
    \mathcal{L}_{\mathrm{dMF}} = \E \left[ \left\| u^{\theta}(z_s, r, s, x_t, t) -\texttt{sg}\!\left( v_s(z_s|x_t, t) - (s-r)\frac{\d}{\d s} u^{\theta}(z_s, r, s, x_t, t) \right) \right\|^2 \right].
\end{equation}
where the expectation in \eqref{condMF} is taken with respect to $x_0, t, x_t, r, s, z_1$. Furthermore, the conditional Mean Flow objective generalizes to 
\begin{equation} \label{condCMF}
    \mathcal{L}_{\mathrm{dCMF}} = \E\left[ \left\| u^{\theta}(z_s, r, s, x_t, t) -\texttt{sg}\!\left( z_1 - x_0 - (s-r)\frac{\d}{\d s} u^{\theta}(z_s, r, s, x_t, t) \right) \right\|^2 \right].
\end{equation}
Using a similar argument as in Proposition~\ref{MFprop}, it is easy to see that the gradients of Equations~\eqref{condMF} and \eqref{condCMF} with respect to the model parameters $\theta$ are the same.

\begin{proposition} \label{dMFprop}
    The loss function $\mathcal{L}_{\mathrm{dMF}}$  in \eqref{condMF} and the loss function   $\mathcal{L}_{\mathrm{dCMF}}$ given in \eqref{condCMF} have the same gradient with respect to the model parameters, that is,  $\nabla_{\theta} \mathcal{L}_{\mathrm{dCMF}} = \nabla_\theta \mathcal{L}_{\mathrm{dMF}}$. 
\end{proposition}

To train the model using \eqref{condCMF}, we only need samples from the joint distribution of $p(x_0, x_t)$. 
Since sampling from $p(x_0|x_t)$ requires simulations of \eqref{sampling_SDE}, we use the reverse kernel $p(x_t|x_0)$, which is known in closed form, along with samples from the data distribution $p(x_0)$.
We present an overview of the training and sampling algorithms for our conditional Mean Flow model in Algorithm~\ref{training_algo} and Algorithm~\ref{inference_algo}, respectively. 

\begin{algorithm}[!ht] 
\caption{STMD training} \label{training_algo} 
\hspace*{\algorithmicindent}\textbf{Input:} Data distribution samples $p_0$, forward SDE scheduler $\{\beta_t,\alpha_t, \sigma_t, \, t\in[0,1]\}$, conditional Mean Flow model parametrization $u^\theta(z, r, s, x, t)$. 
\begin{algorithmic} 
\State Sample $x_0 \sim p_0$, $t \sim \texttt{unif}(0,1)$, $x_t$ through \eqref{noise_kernel}, $z_1 \sim \N(0, I_d)$, $r, s \sim \texttt{sample\_r\_s()}$\footnotemark.
\State $z_s \leftarrow (1-s) x_0 + s z_1$
\State $u^\theta, \frac{\d u^\theta}{\d s} \leftarrow \texttt{jvp}(u^\theta, \, (z_s, r,s, x_t, t), \, (z_1 - x_0), 0, 1, 0, 0)$ 
\State Compute CMF loss $\mathcal{L}_{\mathrm{dCMF}}$ using Equation \eqref{condCMF}. 
\State Take optimization step on $\mathcal{L}_{\mathrm{dCMF}}$. 
\State \Return Trained $u^\theta(z_s, r, s, x_t, t)$
\end{algorithmic} 
\end{algorithm}
\footnotetext{We refer the reader to the Appendix \ref{App_exp_details} for the detailed sampling strategy for $r,s$.}

\begin{algorithm}[!ht] 
\caption{STMD inference }\label{inference_algo} 
\hspace*{\algorithmicindent}\textbf{Input:} Trained model $u^\theta(z_s, r, s, x_t, t)$, initial samples $x_1 \sim \N(0, I_d)$, number of inference steps $n_{\mathrm{inf}}$, number of Mean Flow steps $n_{\mathrm{mf}}$. 
\begin{algorithmic}
\State $\Delta t \leftarrow \frac{1}{n_\mathrm{inf}}$, $\Delta s \leftarrow \frac{1}{n_\mathrm{mf}}$, $t \leftarrow 1$,
\For{$k$ in $\mathtt{range}(n_{\mathrm{inf}})$}:
    \State $z_1 \sim \N(0, I_d), \, s \leftarrow 1$
        \For{$i$ in $\mathtt{range}(n_{\mathrm{mf}})$}: 
            \State $z_{s - \Delta s} \leftarrow z_s - \Delta s \,  u^\theta(z_s, s - \Delta s, s, x_t, t)$
            \State $s \leftarrow s - \Delta s$
        \EndFor
    \State $x_{0} \leftarrow z_0$
    \State $x_{t-\Delta t} \sim p(x_{t- \Delta t}|x_0, x_t)$ from Equation \eqref{bridge_measure}
    \State $t \leftarrow t - \Delta t $
\EndFor
\State
\Return $x_0$
\end{algorithmic} 
\end{algorithm}



\subsection{Convergence Theory}

In this section, we obtain Wasserstein convergence guarantees for the
Mean Flow model, and then generalize them to show that our conditional Mean Flow model converges to the transition kernel of the reverse SDE \eqref{sampling_SDE}. To the best of our knowledge, this is the first result showing convergence of the Mean Flow algorithm in the Wasserstein distance. 

Let $\mathcal{P}_{2}(\mathbb{R}^d)$ be the space of all Borel probability measures on $\R^d$ with finite second moment (with respect to the Euclidean norm). For any two Borel probability measures $\mu,\nu$ in $\mathcal{P}_{2}(\R^d)$, the 2-Wasserstein distance between $\mu$ and $\nu$ is defined as \citep{villani2008optimal}:
$\mathcal{W}_2^2(\mu,\nu):=\inf_{\gamma\in\Pi(\mu,\nu)}\int_{\R^d\times\R^d}\|x-y\|^2\,\d\gamma(x,y)$,
where $\Pi(\mu,\nu)$ is the set of couplings of $(\mu,\nu)$. 

We will first study the convergence of the unconditional Mean Flow model, presented in Section~\ref{sec_meanflow}. 
Recall that sampling using the Mean Flow model can be achieved via:
\begin{equation}
z_{r}=z_{s}-(s-r)u(z_{s},r,s).
\end{equation}
In the case of single-step sampling, we have $z_{0}=z_{1}-u(z_{1},0,1)$ 
with $z_{1}\sim \rho_{1}$.
Next, we obtain the Wasserstein convergence guarantees
for the Mean Flow model for the single-step generator 
$\hat{z}_{0}=z_{1}-\hat{u}(z_{1},0,1)$
with $z_{1}\sim \rho_{1}$, where $\hat{u}$ is a candidate mean velocity field.
First, given the candidate model $\hat{u}$, and for fixed $r, s$ we define the $(r,s)$-conditional Mean Flow error, as 
\begin{equation}
\varepsilon_{r, s} := \E_{z_s}\left[ \left\| \hat u (z_s, r, s) -\texttt{sg}\!\left( v_s(z_s) - (s-r)\frac{\d}{\d s} \hat u (z_s, r, s) \right) \right\|^2 \right].
\end{equation}
Note the Mean Flow loss reduces to $\mathcal{L}_{\mathrm{MF}}= \E_{r,s}[\varepsilon_{r,s}]$. 
We now state our main convergence theorem.

\begin{theorem}\label{thm:mean:flow:convergence}
Let $\hat{\rho}_{0}$ denote the single-step generated distribution 
of $\hat{z}_{0}=z_{1}-\hat{u}(z_{1},0,1)$ with $z_{1}\sim \rho_{1}$, 
where $\hat{u}$ is a candidate mean velocity field. 
Assume that $\E_{s} [ \varepsilon_{0, s}] \leq \varepsilon$. 
Then, 
\begin{equation}
\mathcal{W}_{2}^{2}(\hat{\rho}_{0},\rho_{0})\leq e\varepsilon.  
\end{equation}
\end{theorem}

Notice that the assumption that $\E_{s}[\varepsilon_{0,s}]$ is small is reasonable when training with the loss $\mathcal{L}_{\mathrm{MF}}=\E_{r,s}[\varepsilon_{r,s}]$, since we are minimizing this loss for all values of $r<s$. It is also a standard assumption the literature \cite{Boffi2025flow}. 
Theorem~\ref{thm:mean:flow:convergence} provides the theoretical convergence guarantees in the 2-Wasserstein distance for the Mean Flow model introduced in \cite{geng2025mean}, which bridges a gap between theory and practice. 

We now study the convergence of the conditional Mean Flow model \eqref{cond_meanflow} to the SDE transition kernel $p(x_0|x_t)$ associated with \eqref{sampling_SDE}. 
Let $\hat{p}_{0|t}(x_0|x_t)$ be the single-step generated distribution of 
$\hat{z}_{0}={z}_{1}-\hat{u}({z}_{1},0,1, x_t, t)$, 
where ${z}_{1}\sim \rho_1$ and $\hat{u}$ is a candidate conditional Mean Flow model.
Furthermore, given $\hat{u}=\hat{u}(z_{s},r,s,x_{t},t)$, define the $(t, x_t)$-conditional mean flow error as 
\begin{equation} \label{gamma_t_xt}
    \gamma(r, s, t, x_t) :=   \E_{z_s} \left[ \left\| \hat{u} - \left( v_{s}(z_s| x_t, t) - (s-r)\frac{\d}{\d s} \hat{u} \right) \right\|^2 \right].
\end{equation}
Given \eqref{gamma_t_xt}, the conditional Mean Flow loss \eqref{condMF} equals $\mathcal{L}_{\mathrm{dMF}} = \E_{r,s,t, x_t} [\gamma(r,s,t, x_t)]$.
The following is a corollary of Theorem~\ref{thm:mean:flow:convergence}. 

\begin{corollary} \label{cor:cond_meanflow_convergence}
For a given $t, x_t$, let $\hat{p}_{0|t}(x_0|x_t)$ denote the single-step generated distribution of $\hat{z}_{0}=z_{1}-\hat{u}(z_{1},0,1, x_t, t)$ with $z_{1}\sim \rho_{1}$, where $\hat{u}$ is a candidate conditional mean velocity field. 
Assume that $\E_{s, t, x_t} [ \gamma(0, s, t, x_t)] \leq \varepsilon$. 
Then, 
\begin{equation}
\E_{t, x_t} \left[ \mathcal{W}_{2}^{2}(\hat{p}(x_0|x_t), p(x_0|x_t)) \right] \leq e \varepsilon.
\end{equation}
\end{corollary}

Note that during the inference of the diffusion model using \eqref{sampling_SDE}, we use $x_{1} \sim \mathcal{N}(0,I_d)$
where $\mathcal{N}(0,I_d)$ is a proxy of the density $x_1 \sim {p}_{1}$ given from \eqref{noise_sde}.
This introduces an additional source of error in the generated distribution for $x_0$ in our conditional Mean Flow model. 
In the following proposition, we account for this error.


\begin{proposition}\label{prop:diffusion}
    Let $\hat{p}(x_0|x_t)$ be a candidate conditional Mean Flow model that parametrizes the transition dynamics of the reverse SDE kernel \eqref{sampling_SDE}, and let $\tilde{p}_0 = \int_{\mathbb{R}^{d}} \hat{p}(x_0|x_1) \N(x_1; 0, I_d) \d x_1$ and $\hat{p}_0 = \int_{\mathbb{R}^{d}} \hat{p}(x_0|x_1) p(x_1) \d x_1$ where $p_1(x_1)$ is the density of $x_1$ associated with \eqref{sampling_SDE}. 
    Assuming that $ m_{2} = \E[\|x_0\|^2] < \infty $, we obtain
    \begin{equation}
        \mathcal{W}_2^2 (\hat p_0, \tilde p_0) \leq L^2 \left( \alpha_{1}^{2}m_{2}+(1-\sigma_{1})^{2}d \right), 
    \end{equation}
    where $\alpha_1, \sigma_1$ are defined through \eqref{noise_kernel} and $d$ denotes the problem dimension, provided that the map $\hat{u}(z_1, 0, 1, x_1, 1)$ is $L$-Lipschitz in $x_1$ in expectation over $z_1$, that is, for any $x_1,x'_1$,
\begin{equation}\label{lip_in_x1:main}
    \E_{z_1} [ \| \hat{u}(z_1, 0, 1, x_1, 1) - \hat{u}(z_1, 0, 1, x'_1, 1)\|^2] \leq L^2 \left\| x_1 - x'_1 \right\|^2. 
\end{equation} 
\end{proposition}

In Proposition~\ref{prop:diffusion}, a sufficient condition for \eqref{lip_in_x1:main} is that $\hat u$ be uniformly $L$-Lipschitz in $x_1$ for all $z_1$; however assuming  \eqref{lip_in_x1} directly is a weaker condition. The Lipschitz condition is common in the literature; for example, 
it is often assumed that the candidate velocity field $\hat{v}_{t}$ is Lipschitz in space uniformly in $t$ in flow models \citep{albergo2023building,boffi2025build}. Moreover, the finiteness of the second moment assumption on the data distribution, i.e. $m_{2}=\mathbb{E}\Vert x_{0}\Vert^{2}<\infty$, is also very mild
and commonly assumed in the literature of diffusion models \citep{chen2022improved,chen2022sampling}.
Moreover, we will show that 
if we assume that $\tilde{p}_{1}$ and $\hat{p}_{1}$ have bounded support, which is a common assumption in the setting when the data distribution $p_{0}$ has bounded support, then we can remove the Lipschitz assumption \eqref{lip_in_x1:main}; 
see Proposition~\ref{prop:diffusion:bounded} in Appendix~\ref{App_additional} for details.


Finally, by applying Corollary~\ref{cor:cond_meanflow_convergence} and Proposition~\ref{prop:diffusion}, we obtain the following corollary that provides an upper bound on the 2-Wasserstein distance between the distribution of the output of the conditional Mean Flow model and the true data distribution. 

\begin{corollary}\label{cor:diffusion}
Under the assumptions of Theorem~\ref{thm:mean:flow:convergence}, Corollary~\ref{cor:cond_meanflow_convergence}, 
and Proposition~\ref{prop:diffusion}, 
we have
\begin{align}\label{cor:RHS}
\mathcal{W}_{2}^{2}(\tilde{p}_{0},p_{0})  \leq 2\left(L^2 \left(\alpha_{1}^{2}m_{2}+(1-\sigma_{1})^{2}d\right) + e \varepsilon_1 \right),
\end{align}
provided that $\E_{s, x_1} [\gamma(0, s, 1, x_1)] \leq \varepsilon_1$, where $\alpha_1, \sigma_1$ are defined through \eqref{noise_kernel}. 
\end{corollary}

Note that in Corollary~\ref{cor:diffusion}, the term $\alpha_{1}^{2}m_{2}+(1-\sigma_{1})^{2}d$ and hence the right hand side in \eqref{cor:RHS} can be made arbitrarily small by taking $\alpha_{1}$ to be small (so that $\sigma_{1}=\sqrt{1-\alpha_{1}^{2}}$ is close to $1$); see Corollary~\ref{cor:small} in Appendix~\ref{App_additional} for details.

\section{Related Literature}

\paragraph{Fast deterministic inference and flow-map distillation.}
A large body of work accelerates inference of flow-based models by replacing the expensive numerical integration of the learned flow with learned finite-time maps. 
An early work in this direction is \cite{luhman2021knowledge}, succeeded by Progressive distillation \citep{salimans_progressive_2022}, consistency models \citep{song2023consistency}, consistency trajectory models \citep{kim2023consistency}, flow maps \citep{boffi2025build} and Mean Flows \citep{geng2025mean}, among others. 
They all aim to reduce the number of function evaluations required at inference time, and can be broadly viewed as learning to integrate an underlying ODE of the flow model. 
However, they are primarily designed for deterministic generation, usually through the probability-flow ODE or a deterministic flow model. In contrast, our goal is not only to compress the inference trajectory, but also to preserve the stochastic transition structure of the reverse SDE. STMD therefore learns a conditional stochastic transition map, rather than a deterministic noise-to-data map.

\paragraph{Probabilistic denoising and distributional distillation.}
Several works have recognized that coarse reverse-time stochastic sampling requires more than the conditional mean predicted by standard denoising diffusion models. 
Denoising Diffusion GANs \citep{xiao_tackling_2022} and Adversarial Schr\"odinger Bridge Matching \citep{gushchin2024adversarial} model large transitions with a multimodal conditional GAN, enabling fewer denoising steps by replacing the Gaussian small-step discretization of the inference SDE with a learned stochastic transition. Distributional Diffusion Models with Scoring Rules \citep{de_bortoli_distributional_2025} similarly learn a stochastic denoiser for the full posterior distribution $p(x_0|x_t)$, using proper scoring rules such as energy distance or maximum mean discrepancy. NCVSD \citep{peng_noise_2025} also trains a noise-conditional stochastic denoiser, but does so by distilling a pretrained diffusion teacher through a variational score-distillation objective. 
These works are closest to STMD in target objective, as they also aim to model the distribution of the denoising posterior.
The main distinction is the training principle: DDGAN relies on adversarial training, Distributional Diffusion Models rely on scoring-rule losses, and NCVSD relies on teacher-based score distillation. STMD instead uses a teacher-free, regression-based conditional Mean Flow objective to learn the finite-time transition kernel associated with the reverse SDE.

\paragraph{Distribution matching and bridge distillation.}
A related family of methods distill diffusion or bridge models by matching generated distributions rather than explicitly matching transition kernels. Distribution Matching Distillation \cite{yin2024one} trains a one-step generator by minimizing a KL-based objective between the generated distribution and the distribution induced by a pretrained diffusion model. Inverse Bridge Matching Distillation \citep{gushchin2025inverse} extends this idea to bridge-matching models by learning a stochastic generator whose induced bridge model matches a pretrained bridge vector field. These methods are effective for fast generation, but they primarily match endpoint marginals, rather than learning the conditional transition kernel $p(x_{t'} | x_t)$ of the sampling SDE. 
Furthermore, unlike our approach, they require a pretrained diffusion or bridge matching model during training. 

\paragraph{Transition matching and discrete stochastic generators.}
Focusing on discrete-time inference, recent transition-matching methods also emphasize learning expressive stochastic transitions between intermediate generative states, and have been proposed as flexible alternatives to diffusion and flow matching for efficient and probabilistic generation \citep{shaul2025transition}.
Transition Matching Distillation \citep{nie2026transition} further applies this idea to few-step generation by distilling long diffusion trajectories into a compact transition process. 
These methods share the goal of accelerating inference through learned transitions with STMD, and also share the concept of using a flow model to model coarse transitions.
However, their acceleration principle lies in a patch-based modeling of the transition flow, while the stochasticity of the generator is associated with a bridge model, rather than a diffusion process.
STMD instead uses full-state flow model to parametrize transition, and its acceleration is based on distillation.

\section{Experiments} \label{experiments_section}

We conduct a series of experiments to validate our approach for various image generation and conditional sampling tasks. 
All our models are trained on a single RTX 5090 GPU. 
To allow for fair comparisons when comparing with baselines, we use the same neural network backbone for all approaches and train from scratch for the same number of iterations. 
Although the small number of iterations used in all experiments prevents our networks from achieving state-of-the-art performance, our experiments allow for a comparative study against baselines.

\paragraph{MNIST Example.} 
We first test our method on the MNIST dataset \citep{lecun2010mnist}.
We parametrize the Mean Flow model with a U-Net \citep{ronneberger2015u} using the diffusers toolbox \citep{von-platen-etal-2022-diffusers} and PyTorch \citep{paszke2019pytorch}. 
We provide the exact hyperparameters of the network in Appendix~\ref{App_exp_details}. 
We concatenate the flow variable $z_s$ and the diffusion conditioning $x_t$ along the channel axis and use sinusoidal conditioning and a multilayer ResNet for mixing the conditioning variables $r, s, t$. 
Following \cite{geng2025mean}, instead of feeding all three time steps directly into the conditional Mean Flow network, we found that the model performs best when using $s-r, s, t$ conditioning instead of $r, s, t$ (see inference Algorithm~\ref{inference_algo}). 
We illustrate generated samples using STMD with different numbers of inference steps $n_{\mathrm{mf}}, n_{\mathrm{inf}}$ in Figure~\ref{fig:mnist_samples}. 

To benchmark our model, since the Fr\'{e}chet inception distance (FID) scores~\citep{heusel2017gans} can be unreliable in the MNIST dataset \citep{song_generative_2020}, as common in the literature, we report Fr\'{e}chet distance (FD) calculated in the latent space of a MNIST classifier in Figure~\ref{fig:mnist_fd}.
We defer the details of the classifier used to Appendix~\ref{App_exp_details}. 
Using our FD metric, we compare our model against the original Mean Flow algorithm \citep{geng2025mean} and a Denoising Diffusion Probabilistic Model (DDPM) baseline \citep{ho2020denoising}, while ablating the number of inference steps for each method. We present the results in Figure~\ref{fig:mnist_fd}, where we denote the total number of function evaluations (NFE) during the inference loop of each method on the horizontal axis and the FD score on the vertical axis. We use the same network, trained for the same number of iterations, for all three methods. For the FD calculation, we use 10K generated samples and 10K samples from the test dataset.
\begin{figure}[!ht]
    \centering
    \begin{subfigure}{0.7\linewidth}
        \centering
        \includegraphics[width=1\linewidth]{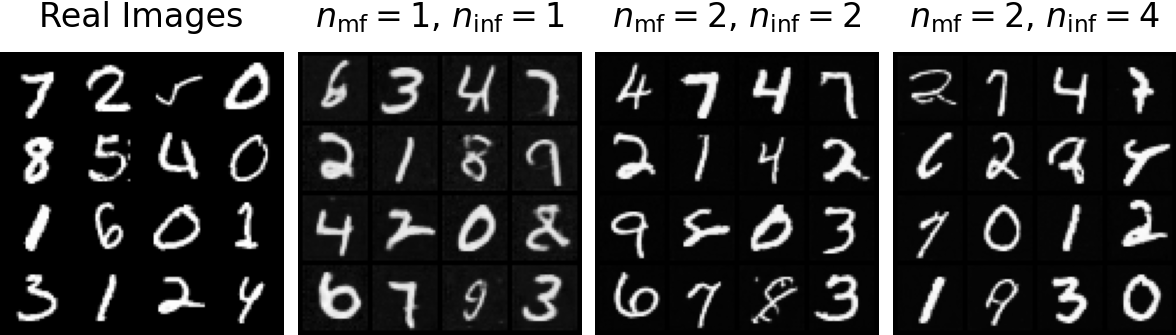}
        \caption{}
        \label{fig:mnist_samples}
    \end{subfigure}
    \hfill
    \begin{subfigure}{0.28 \linewidth}
        \centering
        \includegraphics[width=1\linewidth]{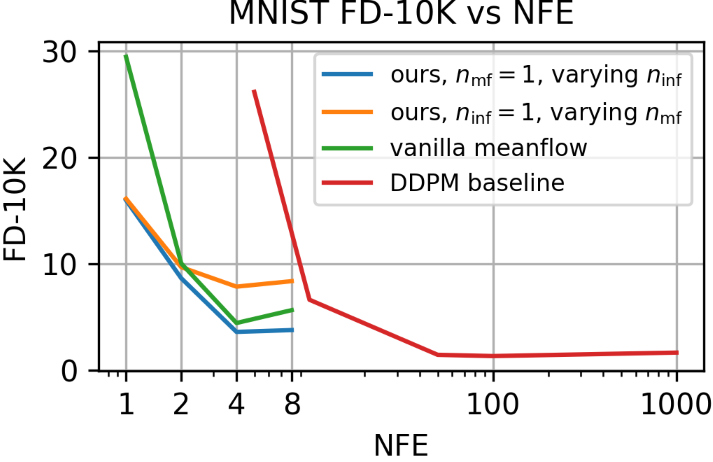}
        \caption{}
        \label{fig:mnist_fd}
    \end{subfigure}
    \caption{(a): Unconditional MNIST samples using various $n_{\mathrm{inf}}, n_{\mathrm{mf}}$ (b): MNIST FD vs NFE.}
    \label{fig:mnist_combined}
\end{figure}
\begin{figure}[!ht]
    \centering
    \begin{subfigure}{0.68\linewidth}
        \centering
        \includegraphics[width=1\linewidth]{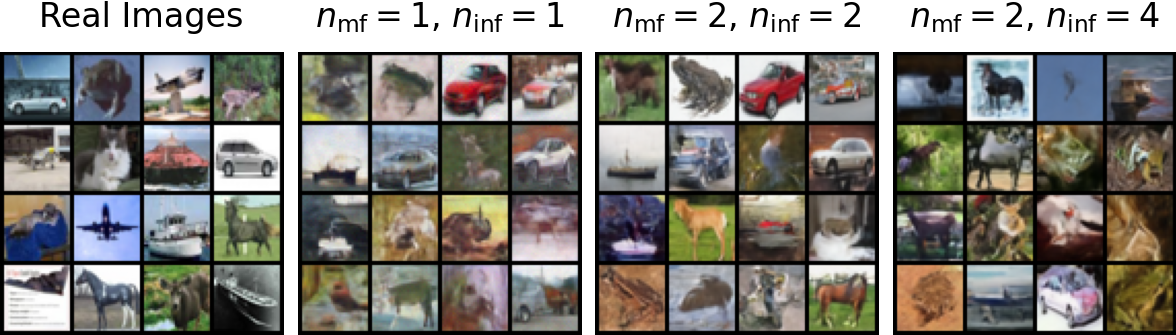}
        \caption{}
        \label{fig:cifar10_samples}
    \end{subfigure}
    \hfill
    \begin{subfigure}{0.3 \linewidth}
        \centering
        \includegraphics[width=1\linewidth]{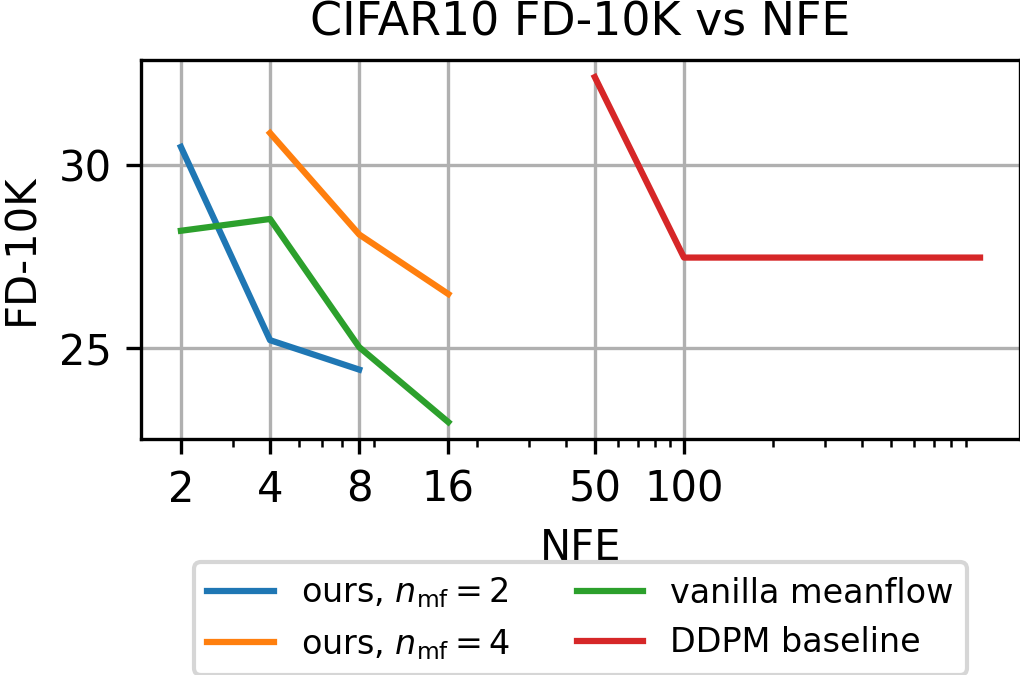}
        \caption{}
        \label{fig:cifar10_fid}
    \end{subfigure}
    \caption{(a): \!\! Unconditional CIFAR10 samples  using various $n_{\mathrm{inf}}, n_{\mathrm{mf}}$ (b): CIFAR10 FID vs NFE.}
    \label{fig:CIFAR10_combined}
\end{figure}
\paragraph{CIFAR10.} 
Next, we test our model in unconditional generation on the CIFAR10 dataset \citep{krizhevsky2009learning}. 
We provide data generated with our STMD approach for a varying number of NFEs in Figure \ref{fig:cifar10_samples}.
Similarly to our MNIST experiment, we used 10K samples from the test set of the dataset to compute FID scores and facilitate computations using the \texttt{torchmetrics} toolkit \citep{detlefsen2022torchmetrics}.
As in the previous example, we compare our model with a DDPM and an original Mean Flow baseline, using the same network and the same number of training iterations for all experiments. We defer the descriptions of these parameters to Appendix~\ref{App_exp_details}. 
We observe that for a fixed number of training iterations, our model performs comparatively with the mean flow model, while retaining full stochasticity of the diffusion-based sampler.
%
\vspace{-2mm}
\paragraph{CelebA.} 
Finally, we conduct experiments using the CelebA \citep{liu2015deep} dataset using a latent DiT architecture \citep{peebles2023scalable}. 
Specifically, we use a DiT/M2 network from \cite{geng2025mean} using an additional temporal conditioning signal for the diffusion time $t$. 
We showcase unconditional samples generated with our model in Figure~\ref{fig:celebA_demo} and an inpainting example in Figure~\ref{fig:inpaint_demo}.
We defer the details of the inpainting algorithm to the Appendix \ref{App_inverse_problems}.
For the unconditional image generation, we used $n_{\mathrm{inf}}=4$ and $n_{\mathrm{mf}}=2$ and calculated an FID score of $8.28$, using 10K samples from the test dataset.
For inpainting, we use $n_{\mathrm{mf}}=2$, and a total of $50$ inference steps.
%
\begin{figure}[!ht]
    \centering
    \includegraphics[width=0.95\linewidth]{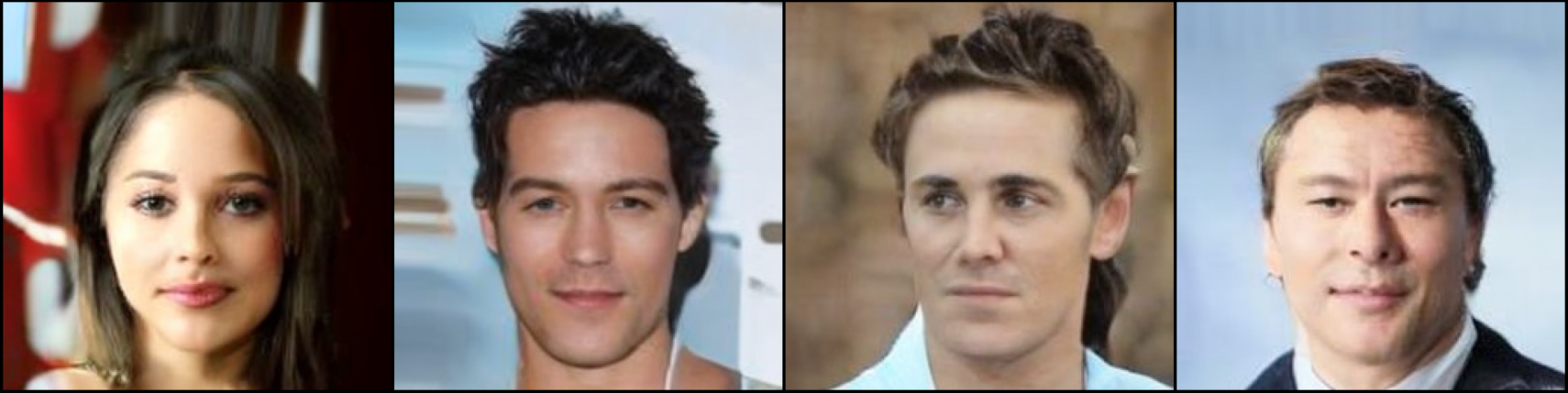}
    \caption{Unconditional generation on the CelebA dataset.}
    \label{fig:celebA_demo}
\end{figure}
\begin{figure}[!ht]
    \centering
    \includegraphics[width=0.95\linewidth, ]{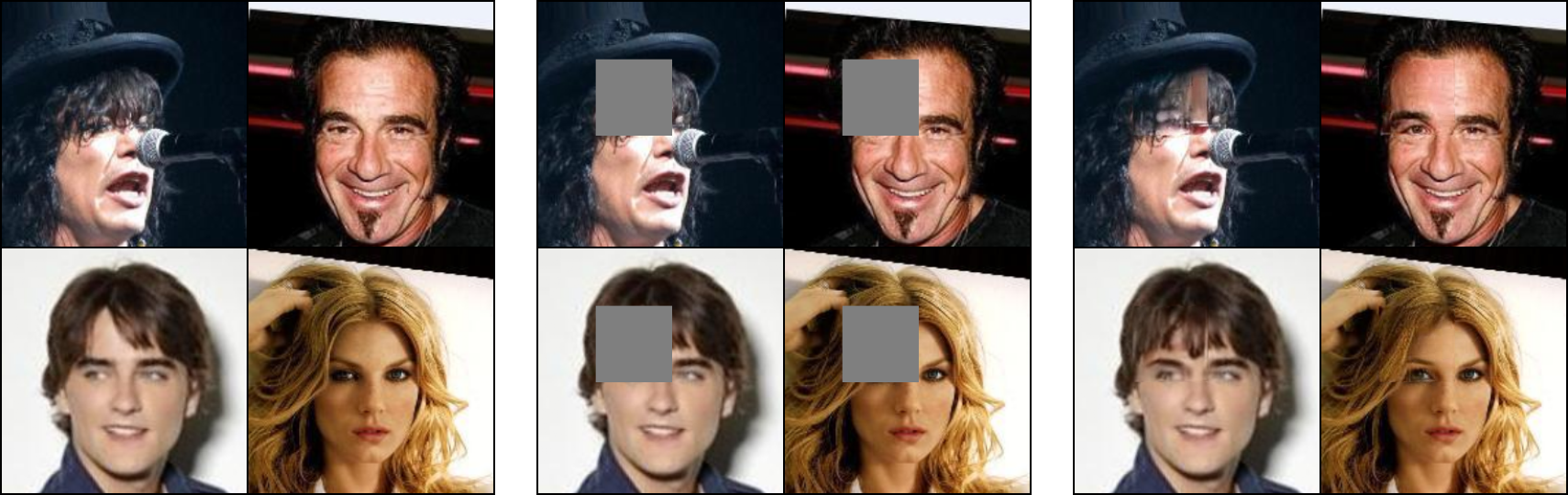}
    \caption{Image inpainting on the CelebA dataset. From left to right: original image, masked image, inpainted image.}
    \label{fig:inpaint_demo}
\end{figure}

\vspace{-3mm}
\section{Conclusions and Limitations}
To conclude, we introduce STMD, a framework for distilling the transition maps of diffusion models into one- or few-step generators. 
Unlike state-of-the-art methods for the same task, our model does not require a pretrained teacher model or bilevel optimization algorithms during training. 
We establish a strong theoretical foundation for our approach and derive a convergence result for the Mean Flow model in the Wasserstein distance. 
We then generalize this result to our STMD model and show that it converges in expectation to the transition kernel of the inference process of the reference sampling diffusion and that the generated sample distribution converges to that of the data. 
Finally, we demonstrate our method on a comprehensive set of image generation and inpainting experiments over the MNIST, CIFAR10, and CelebA datasets and show that our method performs in line with baselines, while allowing for probabilistic sampling and requiring only a small number of inference steps.

\bibliographystyle{icml2026}
\bibliography{refs}

@inproceedings{arjovsky2017wasserstein,
  title={Wasserstein generative adversarial networks},
  author={Arjovsky, Martin and Chintala, Soumith and Bottou, L{\'e}on},
  booktitle={International Conference on Machine Learning},
  pages={214--223},
  volume={70},
  year={2017},
  address={Sydney, {Australia}},
  organization={PMLR}
}

@inproceedings{kingma2014VAE,
  title={Auto-encoding variational {B}ayes},
  author={Kingma, Diederik P and Welling, Max},
  booktitle={International Conference on Learning Representations},
  year={2014}
}

@article{kingma2019introduction,
  title={An introduction to variational autoencoders},
  author={Kingma, Diederik P and Welling, Max},
  journal={Foundations and Trends{\textregistered} in Machine Learning},
  volume={12},
  number={4},
  pages={307--392},
  year={2019},
  publisher={Now Publishers, Inc.}
}

@inproceedings{song2020score,
  title={Score-Based Generative Modeling through Stochastic Differential Equations},
  author={Yang Song and Jascha Sohl-Dickstein and Diederik P Kingma and Abhishek Kumar and Stefano Ermon and Ben Poole},
  booktitle={International Conference on Learning Representations},
  year={2021},
  address={held virtually}
}

@inproceedings{lipman2022flow,
title={Flow Matching for Generative Modeling},
author={Yaron Lipman and Ricky T. Q. Chen and Heli Ben-Hamu and Maximilian Nickel and Matthew Le},
booktitle={International Conference on Learning Representations},
address={Kigali Rwanda},
month={May},
year={2023}
}

@inproceedings{liu2022flow,
  title={Flow Straight and Fast: Learning to Generate and Transfer Data with Rectified Flow},
  author={Liu, Xingchao and Gong, Chengyue and Liu, Qiang},
  booktitle={International Conference on Learning Representations},
  address={Kigali Rwanda},
  month={May},
  year={2023}
}

@inproceedings{paszke2019pytorch,
  title={Py{T}orch: An imperative style, high-performance deep learning library},
  author={Paszke, Adam and Gross, Sam and Massa, Francisco and Lerer, Adam and Bradbury, James and Chanan, Gregory and Killeen, Trevor and Lin, Zeming and Gimelshein, Natalia and Antiga, Luca and Desmaison, Alban and K\"{o}pf, Andreas and Yang, Edward and DeVito, Zach and Raison, Martin and Tejani, Alykhan and Chilamkurthy, Sasank and Steiner, Benoit and Fang, Lu and Bai, Junjie and Chintala, Soumith},
  booktitle={Advances in Neural Information Processing Systems},
  volume={32},
  year={2019}, 
  pages={8026 - 8037},
  publisher = {Curran Associates, Inc.},
  address={Vancouver, Canada}
}

@inproceedings{ho2020denoising,
  title={Denoising diffusion probabilistic models},
  author={Ho, Jonathan and Jain, Ajay and Abbeel, Pieter},
  booktitle={Advances in Neural Information Processing Systems},
  volume={33},
  pages={6840--6851},
  publisher = {Curran Associates, Inc.},
  year={2020}
}

@inproceedings{shi2023diffusion,
  title={Diffusion {Schr{\"o}dinger} bridge matching},
  author={Shi, Yuyang and De Bortoli, Valentin and Campbell, Andrew and Doucet, Arnaud},
  booktitle={Advances in Neural Information Processing Systems},
  volume={36},
  pages = {62183--62223},
  publisher = {Curran Associates, Inc.},
  year={2023}
}

@inproceedings{albergo2023building,
  title={Building Normalizing Flows with Stochastic Interpolants},
  author={Albergo, Michael and Vanden-Eijnden, Eric},
  booktitle={International Conference on Learning Representations},
  address={Kigali Rwanda},
  month={May},
  year={2023}
}

@inproceedings{heusel2017gans,
  title={G{AN}s trained by a two time-scale update rule converge to a local {N}ash equilibrium},
  author={Heusel, Martin and Ramsauer, Hubert and Unterthiner, Thomas and Nessler, Bernhard and Hochreiter, Sepp},
  booktitle={Advances in Neural Information Processing Systems},
  volume={30},
  publisher = {Curran Associates, Inc.},
  year={2017}
}

@inproceedings{shaul2025transition,
 author = {Shaul, Neta and Singer, Uriel and Gat, Itai and Lipman, Yaron},
 booktitle = {Advances in Neural Information Processing Systems},
 editor = {D. Belgrave and C. Zhang and H. Lin and R. Pascanu and P. Koniusz and M. Ghassemi and N. Chen},
 pages = {124347--124388},
 publisher = {Curran Associates, Inc.},
 title = {Transition Matching: Scalable and Flexible Generative Modeling},
 volume = {38},
 year = {2025}
}

@misc{dieleman2024distillation,
  author = {Dieleman, Sander},
  title = {The paradox of diffusion distillation},
  url = {https://sander.ai/2024/02/28/paradox.html},
  year = {2024}
}

@inproceedings{peebles2023scalable,
    address = {Paris, France},
    title = {Scalable {Diffusion} {Models} with {Transformers}},
    urldate = {2025-07-10},
    booktitle = {2023 {IEEE}/{CVF} {International} {Conference} on {Computer} {Vision} ({ICCV})},
    publisher = {IEEE},
    author = {Peebles, William and Xie, Saining},
    month = oct,
    year = {2023},
    pages = {4172--4182}
}

@inproceedings{gushchin2024adversarial,
  title={Adversarial {S}chr{\"o}dinger bridge matching},
  author={Gushchin, Nikita and Selikhanovych, Daniil and Kholkin, Sergei and Burnaev, Evgeny and Korotin, Aleksandr},
  booktitle={Advances in Neural Information Processing Systems},
  volume={37},
  pages={89612--89651},
  publisher = {Curran Associates, Inc.},
  year={2024}
}

@inproceedings{song2023consistency,
    title = {Consistency models},
    author = {Song, Yang and Dhariwal, Prafulla and Chen, Mark and Sutskever, Ilya},
    year = {2023},
    booktitle={Proceedings of the 40th International Conference on Machine Learning},
    volume={202},
    pages={32211-32252},
    organization={PMLR}
}

@inproceedings{
gushchin2025inverse,
title={Inverse Bridge Matching Distillation},
author={Nikita Gushchin and David Li and Daniil Selikhanovych and Evgeny Burnaev and Dmitry Baranchuk and Alexander Korotin},
booktitle={Forty-second International Conference on Machine Learning},
year={2025},
volume={267},
pages={21471-21496},
organization={PMLR}
}

@inproceedings{yin2024one,
  title={One-step diffusion with distribution matching distillation},
  author={Yin, Tianwei and Gharbi, Micha{\"e}l and Zhang, Richard and Shechtman, Eli and Durand, Fredo and Freeman, William T and Park, Taesung},
  booktitle={Proceedings of the IEEE/CVF Conference on Computer Vision and Pattern Recognition},
  pages={6613--6623},
  year={2024}
}

@article{luhman2021knowledge,
  title={Knowledge distillation in iterative generative models for improved sampling speed},
  author={Luhman, Eric and Luhman, Troy},
  journal={arXiv preprint arXiv:2101.02388},
  year={2021}
}

@inproceedings{
goodfellow2014GAN,
title={Generative adversarial nets},
author={Goodfellow, Ian and Pouget-Abadie, Jean and Mirza, Mehdi and Xu, Bing and Warde-Farley, David and Ozair, Sherjil and Courville, Aaron and Bengio, Yoshua},
booktitle={Advances in Neural Information Processing Systems},
year={2014},
volume={27},
publisher = {Curran Associates, Inc.}
}

@article{goodfellow2020generative,
  title={Generative adversarial networks},
  author={Goodfellow, Ian and Pouget-Abadie, Jean and Mirza, Mehdi and Xu, Bing and Warde-Farley, David and Ozair, Sherjil and Courville, Aaron and Bengio, Yoshua},
  journal={Communications of the ACM},
  volume={63},
  number={11},
  pages={139--144},
  year={2020},
  publisher={ACM New York, NY, USA}
}

@inproceedings{kim2023consistency,
  title={Consistency trajectory models: Learning probability flow {ODE} trajectory of diffusion},
  author={Kim, Dongjun and Lai, Chieh-Hsin and Liao, Wei-Hsiang and Murata, Naoki and Takida, Yuhta and Uesaka, Toshimitsu and He, Yutong and Mitsufuji, Yuki and Ermon, Stefano},
  booktitle={International Conference on Learning Representations},
  year={2024}
}

@inproceedings{boffi2025build,
 author = {Boffi, Nicholas and Albergo, Michael and Vanden-Eijnden, Eric},
 booktitle = {Advances in Neural Information Processing Systems},
 editor = {D. Belgrave and C. Zhang and H. Lin and L. Montoya and R. Pascanu and P. Koniusz and M. Ghassemi and N. Chen},
 pages = {33346--33382},
 publisher = {Curran Associates, Inc.},
 title = {How to build a consistency model: Learning flow maps via self-distillation},
 volume = {38},
 year = {2025}
}

@inproceedings{geng2025mean,
    title = {Mean flows for one-step generative modeling},
 author = {Geng, Zhengyang and Deng, Mingyang and Bai, Xingjian and Kolter, Zico and He, Kaiming},
 booktitle = {Advances in Neural Information Processing Systems},
 editor = {D. Belgrave and C. Zhang and H. Lin and L. Montoya and R. Pascanu and P. Koniusz and M. Ghassemi and N. Chen},
 pages = {75460--75482},
 publisher = {Curran Associates, Inc.},
 volume = {38},
 year = {2025}
}

@inproceedings{salimans_progressive_2022,
    title = {Progressive distillation for fast sampling of diffusion models},
    author={Salimans, Tim and Ho, Jonathan},
    booktitle={International Conference on Learning Representations},
    year = {2022}
}

@misc{holderrieth_diamond_2026,
    title = {Diamond maps: Efficient reward alignment via stochastic flow maps},
    shorttitle = {Diamond maps},
    url = {http://arxiv.org/abs/2602.05993},
    doi = {10.48550/arXiv.2602.05993},
    publisher = {arXiv},
    author = {Holderrieth, Peter and Chen, Douglas and Eyring, Luca and Shah, Ishin and Anantharaman, Giri and He, Yutong and Akata, Zeynep and Jaakkola, Tommi and Boffi, Nicholas Matthew and Simchowitz, Max},
    month = feb,
    year = {2026},
    note = {arXiv:2602.05993 [cs]},
}

@inproceedings{de_bortoli_distributional_2025,
    title = {Distributional diffusion models with scoring rules},
    organization = {PMLR},
    author = {De Bortoli, Valentin and Galashov, Alexandre and Guntupalli, J. Swaroop and Zhou, Guangyao and Murphy, Kevin and Gretton, Arthur and Doucet, Arnaud},
    year = {2025},
    volume={267},
    pages={12632-12676},
    booktitle={Proceedings of the 42nd International Conference on Machine Learning}
}

@inproceedings{peng_noise_2025,
    title = {Noise conditional variational score distillation},
    organization = {PMLR},
    author = {Peng, Xinyu and Zheng, Ziyang and Wang, Yaoming and Li, Han and Kan, Nuowen and Dai, Wenrui and Li, Chenglin and Zou, Junni and Xiong, Hongkai},
    year = {2025},
    booktitle = {Proceedings of the 42nd International Conference on Machine Learning},
    volume={267},
    pages = {48923-48947}
}

@inproceedings{chung_diffusion_2024,
    title = {Diffusion posterior sampling for general noisy inverse problems},
    author = {Chung, Hyungjin and Kim, Jeongsol and Mccann, Michael T. and Klasky, Marc L. and Ye, Jong Chul},
    booktitle={International Conference on Learning Representations},
    year = {2023}
}

@article{Efron2011,
  title={Tweedie's formula and seletion bias},
  author={Efron, Bradley},
  journal={Journal of the American Statistical Association},
  volume={106},
  nmber={496},
  pages={1602-1614},
  year={2011}
}

@article{Boffi2025flow,
  title={Flow map matching with stochastic interpolants: A mathematical framework for consistency models},
  author={Boffi, Nicholas M. and Albergo, Michael S. and Vanden-Eijnden, Eric},
  journal={Transactions on Machine Learning Research},
  volume={05},
  pages={1-28},
  year={2025}
}

@book{villani2008optimal,
  title={Optimal Transport: Old and New},
  author={Villani, C{\'e}dric},
  volume={338},
  year={2008},
  publisher={Springer}
}

@inproceedings{Frans2025,
    title = {One step diffusion via shortcut models},
    author = {Frans, Kevin and Hafner, Danijar and Levine, Sergey and Abbeel, Pieter},
    booktitle={International Conference on Learning Representations},
    year = {2025}
}

@inproceedings{Geng-easy-2025,
    title = {Consistency models made easy},
    author = {Geng, Zhengyang and Pokle, Ashwini and Luo, Weijian and Lin, Justin and Kolter, J. Zico},
    booktitle={International Conference on Learning Representations},
    year = {2025}
}

@inproceedings{Lu-2025,
    title = {Simplifying, stabilizing and scaling continuous-time consistency models},
    author = {Lu, Cheng and Song, Yang},
    booktitle={International Conference on Learning Representations},
    year = {2025}
}

@inproceedings{Song-2024,
    title = {Improved techniques for training consistency models},
    author = {Song, Yang and Prafulla Dhariwal},
    booktitle={International Conference on Learning Representations},
    year = {2024}
}

@inproceedings{ronneberger2015u,
  title={U-net: Convolutional networks for biomedical image segmentation},
  author={Ronneberger, Olaf and Fischer, Philipp and Brox, Thomas},
  booktitle={International Conference on Medical Image Computing and Computer-Assisted Intervention},
  pages={234--241},
  year={2015},
  organization={Springer}
}

@misc{von-platen-etal-2022-diffusers,
  author = {Patrick von Platen and Suraj Patil and Anton Lozhkov and Pedro Cuenca and Nathan Lambert and Kashif Rasul and Mishig Davaadorj and Dhruv Nair and Sayak Paul and William Berman and Yiyi Xu and Steven Liu and Thomas Wolf},
  title = {Diffusers: State-of-the-art diffusion models},
  year = {2022},
  publisher = {GitHub},
  journal = {GitHub repository},
  howpublished = {\url{https://github.com/huggingface/diffusers}}
}

@article{lai2025principles,
  title={The principles of diffusion models},
  author={Lai, Chieh-Hsin and Song, Yang and Kim, Dongjun and Mitsufuji, Yuki and Ermon, Stefano},
  journal={arXiv preprint arXiv:2510.21890},
  year={2025}
}

@article{lecun2010mnist,
  title={{MNIST} handwritten digit database},
  author={LeCun, Yann and Cortes, Corinna and Burges, CJ},
  journal={ATT Labs [Online]. Available: http://yann.lecun.com/exdb/mnist},
  volume={2},
  year={2010}
}

@inproceedings{liu2015deep,
  title={Deep learning face attributes in the wild},
  author={Liu, Ziwei and Luo, Ping and Wang, Xiaogang and Tang, Xiaoou},
  booktitle={Proceedings of the IEEE International Conference on Computer Vision},
  pages={3730--3738},
  year={2015}
}

@inproceedings{chen2022improved,
  title={Improved Analysis of Score-based Generative Modeling: User-Friendly Bounds under Minimal Smoothness Assumptions},
  author={Chen, Hongrui and Lee, Holden and Lu, Jianfeng},
  booktitle={International Conference on Machine Learning},
  organization={PMLR},
  volume={202},
  pages={4764-4803},
  year={2023}
}

@inproceedings{chen2022sampling,
  title={Sampling is as easy as learning the score: Theory for diffusion models with minimal data assumptions},
  author={Chen, Sitan and Chewi, Sinho and Li, Jerry and Li, Yuanzhi and Salim, Adil and Zhang, Anru R},
  booktitle={International Conference on Learning Representations},
  year={2023}
}

@article{Anderson1982,
	Author = {Anderson, Brian D. O.},
	Journal = {Stochastic Processes and their Applications},
	Number = {3},
	Pages = {313-326},
	Title = {Reverse-time diffusion equation models},
	Volume = {12},
	Year = {1982}}

@book{Schervish1995,
  title={Theory of Statistics},
  author={Schervish, Mark J.},
  publisher={Springer},
  year={1995}
}

@Techreport{krizhevsky2009learning,
 author = {Krizhevsky, Alex and Hinton, Geoffrey},
 address = {Toronto, Ontario},
 institution = {University of Toronto},
 publisher = {Technical report, University of Toronto},
 title = {Learning multiple layers of features from tiny images},
 year = {2009},
 title_with_no_special_chars = {Learning multiple layers of features from tiny images},
 url = {https://www.cs.toronto.edu/~kriz/learning-features-2009-TR.pdf}
}

@inproceedings{song_generative_2020,
 author = {Song, Yang and Ermon, Stefano},
 booktitle = {Advances in Neural Information Processing Systems},
 editor = {H. Wallach and H. Larochelle and A. Beygelzimer and F. d\textquotesingle Alch\'{e}-Buc and E. Fox and R. Garnett},
 publisher = {Curran Associates, Inc.},
 title = {Generative Modeling by Estimating Gradients of the Data Distribution},
 volume = {32},
 year = {2019}
}

@article{GibbsSu2002,
	Author = {Gibbs, Alison L. and Su, Francis Edward},
	Journal = {International Statistical Review},
	Pages = {419--435},
	Title = {On choosing and bounding probability metrics},
	Volume = {70},
    Number={3},
	Year = {2002}}

@inproceedings{xiao_tackling_2022,
    title = {Tackling the {Generative} {Learning} {Trilemma} with {Denoising} {Diffusion} {GANs}},
    author = {Xiao, Zhisheng and Kreis, Karsten and Vahdat, Arash},
    year = {2022},
    booktitle={International Conference on Learning Representations}
}

@article{nie2026transition,
  title={Transition Matching Distillation for Fast Video Generation},
  author={Nie, Weili and Berner, Julius and Ma, Nanye and Liu, Chao and Xie, Saining and Vahdat, Arash},
  journal={arXiv preprint arXiv:2601.09881},
  year={2026}
}

@book{Bakry2014,
	Author = {Bakry, Dominique and Gentil, Ivan and Ledoux, Michel},
	Volume={348},
	publisher={Springer},
    address={Cham},
	Title = {Analysis and Geometry of Markov Diffusion Operators},
	Year = {2014}}

@InProceedings{Follmer1985,
author={F{\"o}llmer, Hans},
editor={Metivier, M.
and Pardoux, E.},
title={An entropy approach to the time reversal of diffusion processes},
booktitle={Stochastic Differential Systems Filtering and Control},
year={1985},
publisher={Springer},
address={Berlin, Heidelberg},
pages={156-163}
}

@article{Cattiaux2023,
	Author ={Cattiaux, Patrick and Conforti, Giovanni and Gentil, Ivan and L\'{e}onard, Christian},
	Journal = {Annales de l'Institut Henri Poincar\'{e}, Probabilit\'{e}s et Statistiques},
	Pages = {1844-1881},
	Title = {Time reversal of diffusion processes under a finite entropy condition},
	Volume = {59},
    Number={4},
	Year = {2023}}

@inproceedings{graikos_fast_2025,
 author = {Graikos, Alexandros and Jojic, Nebojsa and Samaras, Dimitris},
 booktitle = {Advances in Neural Information Processing Systems},
 editor = {D. Belgrave and C. Zhang and H. Lin and R. Pascanu and P. Koniusz and M. Ghassemi and N. Chen},
 pages = {48205--48240},
 publisher = {Curran Associates, Inc.},
 title = {Fast constrained sampling in pre-trained diffusion models},
 volume = {38},
 year = {2025}
}

@article{detlefsen2022torchmetrics,
  title   = {Torch{M}etrics: Measuring Reproducibility in {P}y{T}orch},
  author  = {Detlefsen, Nicki Skafte and Borovec, Jiri and Schock, Justus and Harsh, Ananya and Koker, Teddy and Di Liello, Luca and Stancl, Daniel and Quan, Changsheng and Grechkin, Maxim and Falcon, William},
  journal = {Journal of Open Source Software},
  year    = {2022},
  url     = {https://github.com/Lightning-AI/torchmetrics}
}

@inproceedings{wang2022zero,
  title={Zero-shot image restoration using denoising diffusion null-space model},
  author={Wang, Yinhuai and Yu, Jiwen and Zhang, Jian},
  booktitle={International Conference on Learning Representations},
  year={2023}
}


\appendix

\begin{center}
\Large \bf Stochastic Transition-Map Distillation for\\
Fast Probabilistic Inference \vspace{3pt}\\ {\normalsize Supplementary Material}
\end{center}

The supplementary material is organized as follows.
\begin{itemize}
\item
In Appendix~\ref{App_proofs}, we provide the technical proofs of the main results in the paper.
\item 
In Appendix~\ref{App_additional}, we present additional technical results and their proofs.
\item 
In Appendix~\ref{App_inverse_problems}, we discuss inverse problems using STMD.
\item 
In Appendix~\ref{App_exp_details}, we provide additional details about our numerical experiments.
\end{itemize}

\section{Proofs of the Main Results}\label{App_proofs}

\subsection{Proof of Proposition~\ref{MFprop}}

\begin{proof}
    Let $v_s(z_s) = \E[z_1 - z_0|z_s]$.  
    Starting with Equation \eqref{CMF}, we obtain
    \begin{align*}
        \mathcal{L}_{\mathrm{CMF}} & = \E\left[ \left\| u^{\theta}(z_s, r, s) -\texttt{sg}\!\left( (z_1 - z_0) - v_s(z_s) + v_s(z_s) - (s-r)\frac{\d}{\d s} u^{\theta}(z_s, r, s) \right) \right\|^2 \right] \\
        & = \E\left[ \left\| u^{\theta}(z_s, r, s) -\texttt{sg}\!\left( v_s(z_s) - (s-r)\frac{\d}{\d s} u^{\theta}(z_s, r, s) \right) \right\|^2 \right] \\
        & \quad+ \E\left[ \left\| (z_1 - z_0) - v_s(z_s) \right\|^2 \right] \\
        &\quad\quad + 2 \E \left[ \left\langle u^{\theta}(z_s, r, s) -\texttt{sg}\!\left( v_s(z_s) - (s-r)\frac{\d}{\d s} u^{\theta}(z_s, r, s) \right), (z_1 - z_0) - v_s(z_s) \right\rangle \right] \\
        & = \mathcal{L}_{\mathrm{MF}} + \E\left[ \left\| (z_1 - z_0) - v_s(z_s) \right\|^2 \right].
    \end{align*}
    By taking the gradient on both sides, and observing that the second term in the final equation does not depend on the model parameters, we obtain $\nabla_{\theta} \mathcal{L}_{\mathrm{CMF}} = \nabla_\theta \mathcal{L}_{\mathrm{MF}}$, which concludes the proof. 
\end{proof}


\subsection{Proof of Proposition~\ref{dMFprop}}

\begin{proof}
The proof follows the similar steps
as in the proof of Proposition~\ref{MFprop}.
    Let $v_s(z_s|x_t, t) = \E[z_1 - x_0|z_s, x_t, t]$.   
    Starting with Equation \eqref{condCMF}, we obtain
    \begin{align*}
        \mathcal{L}_{\mathrm{dCMF}} & = \E\Bigg[ \Bigg\| u^{\theta}(z_s, r, s,x_{t},t)\nonumber
        \\
        &\qquad\quad-\texttt{sg}\!\left( (z_1 - x_0) - v_s(z_s|x_t, t) + v_s(z_s|x_t, t) - (s-r)\frac{\d}{\d s} u^{\theta}(z_s, r, s,x_{t},t) \right) \Bigg\|^2 \Bigg] \\
        & = \E\left[ \left\| u^{\theta}(z_s, r, s,x_{t},t) -\texttt{sg}\!\left( v_s(z_s|x_{t},t) - (s-r)\frac{\d}{\d s} u^{\theta}(z_s, r, s,x_{t},t) \right) \right\|^2 \right] \\
        & \qquad+ \E\left[ \left\| (z_1 - x_0) - v_s(z_s|x_{t},t) \right\|^2 \right]. \\
        &\qquad\qquad + 2\E \Bigg[ \Bigg\langle u^{\theta}(z_s, r, s) -\texttt{sg}\!\left( v_s(z_s|x_{t},t) - (s-r)\frac{\d}{\d s} u^{\theta}(z_s, r, s) \right), 
        \nonumber
        \\
        &\qquad\qquad\qquad\qquad\qquad\qquad\qquad\qquad\qquad\qquad\qquad
        (z_1 - x_0) - v_s(z_s|x_{t},t) \Bigg\rangle \Bigg] \\
        & = \mathcal{L}_{\mathrm{dMF}} + \E\left[ \left\| (z_1 - x_0) - v_s(z_s|x_{t},t) \right\|^2 \right].
    \end{align*}
    By taking the gradient on both sides, and observing that the second term in the final equation does not depend on the model parameters, we obtain $\nabla_{\theta} \mathcal{L}_{\mathrm{dCMF}} = \nabla_\theta \mathcal{L}_{\mathrm{dMF}}$, which concludes the proof. 
\end{proof}

\subsection{Proof of Theorem~\ref{thm:mean:flow:convergence}}

\begin{proof}
Let $z_{0}=z_{1}-u(z_{1},0,1)$
and $\hat{z}_{0}=z_{1}-\hat{u}(z_{1},0,1)$
where $z_{1}\sim \rho_{1}$. 
Then $z_{0}\sim \rho_{0}$ and $\hat{z}_{0}\sim \hat{\rho}_{0}$. 

By the definition of the 2-Wasserstein distance, 
we have the bound:
\begin{equation}\label{to:bound}
\mathcal{W}_{2}^{2}(\hat{\rho}_{0},\rho_{0})
\leq\mathbb{E}\left\Vert u(z_{1},0,1)-\hat{u}(z_{1},0,1)\right\Vert^{2}.
\end{equation}
To upper bound the right-hand side of \eqref{to:bound}, let us introduce:
\begin{equation}
E_{r,s}:=\mathbb{E}_{z_{s}}\left\Vert\left(z_{s}-(s-r)u(z_{s},r,s)\right)-\left(z_{s}-(s-r)\hat{u}(z_{s},r,s)\right)\right\Vert^{2},
\end{equation}
for any $0 \leq r  < s \leq 1$. 
Then, it follows that
\begin{equation}
E_{r,s}=\mathbb{E}_{z_{s}}\left\Vert (r-s)(u(z_{s},r,s)-\hat{u}(z_{s},r,s))\right\Vert^{2},
\end{equation}
and the right hand side of \eqref{to:bound} is simply $E_{0,1}$. 
Next, we can compute that for any $0 \leq r  < s \leq 1$,
\begin{align}
\frac{\partial}{\partial s}
E_{r,s}
&=2\mathbb{E}_{z_{s}}\left\langle (r-s)(u(z_{s},r,s)-\hat{u}(z_{s},r,s)),\frac{\d}{\d s}(r-s)u(z_{s},r,s)-\frac{\d}{\d s}(r-s)\hat{u}(z_{s},r,s)\right\rangle
\nonumber
\\
&=2\mathbb{E}_{z_{s}}\left\langle (r-s)(u(z_{s},r,s)-\hat{u}(z_{s},r,s)),-u(z_{s},r,s)+\hat{u}(z_{s},r,s)\right\rangle
\nonumber
\\
&\quad
+2\mathbb{E}_{z_{s}}\left\langle (r-s)(u(z_{s},r,s)-\hat{u}(z_{s},r,s)),(r-s)\frac{\d u(z_{s},r,s)}{\d s}-(r-s)\frac{\d\hat{u}(z_{s},r,s)}{\d s}\right\rangle
\nonumber
\\
&=2\mathbb{E}_{z_{s}}\left\langle (r-s)(u(z_{s},r,s)-\hat{u}(z_{s},r,s)),-u(z_{s},r,s)+\hat{u}(z_{s},r,s)\right\rangle
\nonumber
\\
&\quad
+2\mathbb{E}_{z_{s}}\left\langle (r-s)(u(z_{s},r,s)-\hat{u}(z_{s},r,s)),u(z_{s},r,s)-v_{s}(z_{s})-(r-s)\frac{\d}{\d s}\hat{u}(z_{s},r,s)\right\rangle
\nonumber
\\
&=2\mathbb{E}_{z_{s}}\left\langle (r-s)(u(z_{s},r,s)-\hat{u}(z_{s},r,s)),\hat{u}(z_{s},r,s)-v_{s}(z_{s})-(r-s)\frac{\d}{\d s}\hat{u}(z_{s},r,s)\right\rangle
\nonumber
\\
&\leq
\mathbb{E}_{z_{s}}\Vert (r-s)(u(z_{s},r,s)-\hat{u}(z_{s},r,s))\Vert^{2}
\nonumber
\\
&\qquad\qquad\qquad
+\mathbb{E}_{z_{s}}\left\Vert\hat{u}(z_{s},r,s)-v_{s}(z_{s})-(r-s)\frac{\d}{\d s}\hat{u}(z_{s},r,s)\right\Vert^{2}
\nonumber
\\
&=E_{r,s}+\mathbb{E}_{z_{s}}\left\Vert\hat{u}(z_{s},r,s)-v_{s}(z_{s})-(r-s)\frac{\d}{\d s}\hat{u}(z_{s},r,s)\right\Vert^{2}
\nonumber
\\
&=
E_{r,s}+\varepsilon_{r,s}.\label{key:ineq:1}
\end{align}
By product rule and \eqref{key:ineq:1}, we have
\begin{align} \label{key:ineq:2}
\frac{\partial}{\partial s}\left(e^{-s}E_{r,s}\right)
=-e^{-s}E_{r,s}+e^{-s}\frac{\partial}{\partial s}E_{r,s}
\leq e^{-s}\varepsilon_{r,s}.
\end{align}
By letting $r=0$ and integrating $s$ from $0$ to $1$ in \eqref{key:ineq:2}, we get
\begin{align}
e^{-1}E_{0,1}-e^{-1}E_{0,0}
=e^{-1}E_{0,1}\leq\int_{0}^{1}e^{-s}\varepsilon_{0,s}\mathrm{d}s
\leq 
\int_{0}^{1}\varepsilon_{0,s}\mathrm{d}s
=\mathbb{E}_{s}[\varepsilon_{0,s}]
\leq
\varepsilon,
\end{align}
where we used the fact that $E_{0,0}=0$, which implies that
\begin{equation}
E_{0,1}\leq e\varepsilon.
\end{equation}
This completes the proof. 
\end{proof}

\subsection{Proof of Corollary~\ref{cor:cond_meanflow_convergence}}

\begin{proof}

Let $x_0 = z_{0}=z_{1}-u(z_{1},0,1, x_t, t)$
and $\hat{x}_0 = \hat{z}_{0}=z_{1}-\hat{u}(z_{1},0,1, x_t, t)$
where $z_{1}\sim \rho_{1}$.
Then $x_{0}\sim p_{0|t}$ and $\hat{x}_{0}\sim \hat{p}_{0|t}$. 



First, given a conditional Mean Flow model $\hat{u}=\hat{u}(z_s, r, s, x_t, t)$, recall that
\begin{equation} 
    \gamma(r, s, t, x_t) =  \E_{z_{s}} \left[ \left\| \hat{u} - \left( v_{s}(z_s| x_t, t) - (s-r)\frac{\d}{\d s} \hat{u} \right) \right\|^2\right].  
\end{equation}
Under our assumption, 
$\E_{s, t, x_t} [ \gamma(0, s, t, x_t)] \leq \varepsilon<\infty$,
which implies that
$\mathbb{E}_{s}[\gamma(0,s,t, x_t)]$ is finite for a.e. $t,x_{t}$. Using Theorem~\ref{thm:mean:flow:convergence}, we obtain for a.e. $t,x_{t}$, 
\begin{equation}\label{ineq:conditional}
    \mathcal{W}^2_2 ( p_{0|t}, \hat{p}_{0|t} ) \leq e\mathbb{E}_{s}[\gamma(0,s,t,x_{t})].
\end{equation}

By taking the expectation on both sides, we obtain 
\begin{equation}
    \E_{t, x_t} \left[ \mathcal{W}^2_2 ( p_{0|t}, \hat{p}_{0|t} ) \right] \leq e \varepsilon,
\end{equation}
which concludes the proof. 
\end{proof}


\subsection{Proof of Proposition \ref{prop:diffusion}}

\begin{proof}
Consider the coupling
\begin{align}
& x_{1}=\alpha_{1}x_{0}+\sigma_{1}\epsilon\sim p_{1},\label{coupling:1}
\\
&\tilde{x}_{1}=\epsilon\sim\tilde{p}_{1} = \mathcal{N}(0,I_{d}), \label{coupling:2}
\end{align}
where $\epsilon$ is independent of $x_{0}$, $\alpha_{1}^{2}=\exp\left(-\int_{0}^{1}\beta_{\tau}\mathrm{d}\tau\right)$ and $\sigma_{1}=\sqrt{1-\alpha_{1}^{2}}$.
Given a candidate conditional one-step Mean Flow model $\hat{u}$, define 
\begin{align}
    & \hat{x}_0 = z_1 - \hat{u}(z_1, 0, 1, x_1, 1), \qquad z_1 \sim \N(0, I_d), \quad x_1 \sim {p}_1,  \\
    & \tilde{x}_0 = z_1 - \hat{u}(z_1, 0, 1, \tilde x_1, 1), \qquad z_1 \sim N(0, I_d),\quad \tilde x_1 \sim \tilde{p}_1,
\end{align}
where $z_1$ is independent of $\epsilon, x_0, x_1, \tilde x_1$, and let $\hat p_0, \tilde p_0$ denote distributions for $\hat{x}_0, \tilde x_0$ respectively.
Furthermore, under our assumption, the map $\hat{u}(z_1, 0, 1, x_1, 1)$ is $L$-Lipschitz in $x_1$ in expectation over $z_1$, that is, for any $x_1,x'_1$, 
\begin{equation} \label{lip_in_x1}
    \E_{z_1} \left[ \left\| \hat{u}(z_1, 0, 1, x_1, 1) - \hat{u}(z_1, 0, 1, x'_1, 1)\right\|^2\right] \leq L^{2} \left\| x_1 - x'_1 \right\|^2. 
\end{equation} 

By the definition of the 2-Wasserstein distance, we obtain 
\begin{align}
\mathcal{W}_{2}^{2}(\hat p_{0},\tilde{p}_{0}) & \leq \E \left\| \hat x_0 - \tilde x_0 \right\| ^2   \nonumber \\ 
& \leq \E \left\| \hat{u}(z_1, 0, 1, x_1, 1) - \tilde{u}(z_1, 0, 1, \tilde x_1, 1) \right\|^2   \nonumber \\
& \leq L^2 \, \E \left\| x_1 - \tilde x_1\right\|^2  \label{trick_1}  \\
& = L^2 \mathbb{E}\left\|\alpha_{1}x_{0}+(1-\sigma_{1})\epsilon \right\|^{2}   \nonumber \\
& = L^2 \left( \alpha_{1}^{2}\mathbb{E}\Vert x_{0}\Vert^{2}+(1-\sigma_{1})^{2}d \right)  \label{trick_2} \\
& \leq L^2 \left( \alpha_{1}^{2}m_{2}+(1-\sigma_{1})^{2}d \right), \label{conclude:1}
\end{align}
where in \eqref{trick_1} we used the Lipschitz-in-$x_1$ property of the flow map, and \eqref{trick_2} used the independence of $\epsilon\sim\mathcal{N}(0,I_{d})$ and $x_{0}$ and the finiteness of $m_{2}=\mathbb{E}\Vert x_{0}\Vert^{2}$.

Hence, we conclude from \eqref{conclude:1} that 
\begin{align}
\mathcal{W}_{2}^{2}(\hat{p}_{0}, \tilde{p}_{0})
\leq
L^2\left(\alpha_{1}^{2}m_{2}+(1-\sigma_{1})^{2}d\right).
\end{align}
This completes the proof.
\end{proof}

\subsection{Proof of Corollary~\ref{cor:diffusion}}
By applying the triangle inequality for the 2-Wasserstein distance, Young's inequality and Proposition~\ref{prop:diffusion},
we have
\begin{align}
\mathcal{W}_{2}^{2}(\tilde{p}_{0},p_{0})
&\leq 2\left(\mathcal{W}_{2}^{2}(\tilde{p}_{0},\hat{p}_{0})+\mathcal{W}_{2}^{2}(\hat{p}_{0},p_{0})\right)
\nonumber\\
& = 2\left(L^2 (\alpha_{1}^{2}m_{2}+(1-\sigma_{1})^{2}d) +\mathcal{W}_{2}^{2}(\hat{p}_{0},p_{0})\right). \label{step1} 
\end{align}

Since, by assumption, $\E_{s, x_1} \left[ \gamma(0, s, 1, x_1) \right] \leq \varepsilon_1$, we obtain 
\begin{align}
    \mathcal{W}_{2}^{2}(\hat{p}_{0},p_{0}) & = \mathcal{W}_{2}^{2}\left(\E_{x_1} \left[\hat{p}_{0|1}(x_0|x_1)\right],\E_{x_1} \left[p_{0|1}(x_0|x_1)\right]\right)\nonumber \\
    & \leq \E_{x_1} \Big[ \mathcal{W}_{2}^2 \big(\hat{p}_{0|1}(x_0|x_1), p_{0|1}(x_0|x_1) \big)\Big]\nonumber \\
    & \leq \E_{s, x_1} \left[ e \gamma(0, s, 1, x_1)\right]\nonumber 
    \\
    &\leq e \varepsilon_1, \label{step2}
\end{align}
where we applied \eqref{ineq:conditional}.

Combining Equations \eqref{step1} and \eqref{step2}, we obtain 
\begin{equation}
    \mathcal{W}_{2}^{2}(\tilde{p}_{0},p_{0}) \leq 2\left(L^2 \left(\alpha_{1}^{2}m_{2}+(1-\sigma_{1})^{2}d\right) + e\varepsilon_1 \right),
\end{equation}
which concludes the proof. 



\section{Additional Technical Results}\label{App_additional}

Proposition~\ref{prop:diffusion} relies
on a Lipschitz assumption \eqref{lip_in_x1:main}. 
We will show that 
if we assume $\tilde{p}_{1}$ and $\hat{p}_{1}$ have bounded support, which naturally arises
in the setting when the data distribution $p_{0}$ has bounded support, then we can remove this
Lipschitz assumption.

\begin{proposition}\label{prop:diffusion:bounded}
    Let $\hat{p}(x_0|x_t)$ be a candidate conditional Mean Flow model that parametrizes the transition dynamics of the reverse SDE kernel \eqref{sampling_SDE}, and let $\tilde{p}_1 = \int_{\mathbb{R}^{d}} \hat{p}(x_0|x_1) \N(x_1; 0, I_d) \d x_1$ and $\hat{p}_1 = \int_{\mathbb{R}^{d}} \hat{p}(x_0|x_1) p(x_1) \d x_1$ where $p_1(x_1)$ is the density of $x_1$ associated with \eqref{sampling_SDE}. 
    Assuming that $\tilde{p}_{1}$ and $\hat{p}_{1}$ have bounded support, which is contained in a Euclidean ball centered at the origin with radius $R>0$ in $\mathbb{R}^{d}$, we obtain
\begin{equation}
        \mathcal{W}_2^2 (\hat p_1, \tilde p_1)
        \leq
        2\sqrt{2\KL\left(p_{0}\Vert\mathcal{N}(0,I_{d})\right)}R^{2}\alpha_{1},
\end{equation}
where $\alpha_1$ is defined through \eqref{noise_kernel}. 
\end{proposition}

\begin{proof}
Under our assumption, $\tilde{p}_{1}$ and $\hat{p}_{1}$ have bounded support, which is contained in a Euclidean ball centered at the origin with radius $R>0$ in $\mathbb{R}^{d}$. 
Then, the 2-Wasserstein distance
can be upper bounded by the total variation (TV) distance \citep{GibbsSu2002}:
\begin{equation}\label{apply:1}
\mathcal{W}_{2}(\hat{p}_{1},\tilde{p}_{1})
\leq 2R\sqrt{\mathrm{TV}(\hat{p}_{1},\tilde{p}_{1})}.
\end{equation}
Next, we recall that
\begin{align}
&\tilde{p}_1 = \int_{\mathbb{R}^{d}} \hat{p}(x_0|x_1) \N(x_1; 0, I_d) \d x_1,
\\
&\hat{p}_1 = \int_{\mathbb{R}^{d}} \hat{p}(x_0|x_1) p(x_1) \d x_1,
\end{align}
where $p_1(x_1)$ is the density of $x_1$ associated with \eqref{sampling_SDE}. 

By data processing inequality, we get
\begin{equation}\label{apply:2}
\mathrm{TV}(\hat{p}_{1},\tilde{p}_{1})
\leq
\mathrm{TV}\left(\N(x_1; 0, I_d),p(x_{1})\right).
\end{equation}
Moreover, by Pinsker's inequality, the TV distance can be upper bounded using the KL divergence as follows:
\begin{equation}\label{apply:3}
\mathrm{TV}\left(\N(x_1; 0, I_d),p(x_{1})\right)
\leq\sqrt{\frac{1}{2}\KL\left(p(x_{1})\Vert\N(x_1; 0, I_d)\right)}.
\end{equation}
Note that $p(x_{1})$ denotes the distribution 
of $x_{1}$, where $x_t$ follows the Ornstein-Uhlenbeck process \eqref{noise_sde}:
\begin{equation}
    \d x_t = - \frac{1}{2} \beta_t x_t \, \d t + \sqrt{\beta_t} \, \d w_t, \quad x_0 \sim p_0.
\end{equation}
By a deterministic time change, one can see
that $x_{t}=y_{\tau(t)}$ in distribution, where
$y_{t}$ follows the rescaled Ornstein-Uhlenbeck process:
\begin{equation}
    \d y_t = -  y_t \, \d t + \sqrt{2} \, \d w_t, \quad y_0 \sim p_0,    
\end{equation}
with
\begin{equation}
\tau(t):=\frac{1}{2}\int_{0}^{t}\beta_{\tau}\d\tau.
\end{equation}
It follows from the proof of Theorem~2 in \cite{chen2022sampling} and Theorem 5.2.1. in \cite{Bakry2014} that 
\begin{equation}
\KL\left(\mathrm{Law}(y_{t})\Vert\mathcal{N}(0,I_{d})\right)
\leq e^{-2t}\KL\left(\mathrm{Law}(y_{0})\Vert\mathcal{N}(0,I_{d})\right)
=e^{-2t}\KL\left(p_{0}\Vert\mathcal{N}(0,I_{d})\right).
\end{equation}
Therefore, we have
\begin{equation}
\KL\left(\mathrm{Law}(x_{t})\Vert\mathcal{N}(0,I_{d})\right)
\leq e^{-2\tau(t)}\KL\left(p_{0}\Vert\mathcal{N}(0,I_{d})\right).
\end{equation}
In particular, by letting $t=1$ with $x_{1}\sim p_{1}$, we obtain
\begin{equation}\label{apply:4}
\KL\left(p_{1}\Vert\mathcal{N}(0,I_{d})\right)
\leq e^{-2\tau(1)}\KL\left(p_{0}\Vert\mathcal{N}(0,I_{d})\right)
=\alpha_{1}^{2}\KL\left(p_{0}\Vert\mathcal{N}(0,I_{d})\right),
\end{equation}
where $\alpha_1$ is defined through \eqref{noise_kernel}. 

Hence, by combining \eqref{apply:1}, \eqref{apply:2}, \eqref{apply:3} and \eqref{apply:4}, we conclude that
\begin{equation}
\mathcal{W}_2^2 (\hat p_1, \tilde p_1)
\leq
2\sqrt{2}R^{2}\alpha_{1}\left(\KL\left(p_{0}\Vert\mathcal{N}(0,I_{d})\right)\right)^{1/2}.
\end{equation}
This completes the proof.
\end{proof}

By applying Corollary~\ref{cor:cond_meanflow_convergence} and Proposition~\ref{prop:diffusion:bounded}, we obtain the following corollary which is an analogue of Corollary~\ref{cor:diffusion}. 

\begin{corollary}\label{cor:diffusion:bounded}
Under the assumptions of Theorem~\ref{thm:mean:flow:convergence}, Corollary~\ref{cor:cond_meanflow_convergence}, 
and Proposition~\ref{prop:diffusion:bounded}, 
we have
\begin{align}\label{cor:bounded:RHS}
\mathcal{W}_{2}^{2}(\tilde{p}_{0},p_{0})  \leq 2\left(2\sqrt{2\kappa_{0}}R^{2}\alpha_{1} + e\varepsilon_1 \right),
\end{align}
provided that $\E_{s, x_1} [\gamma(0, s, 1, x_1)] \leq \varepsilon_1$, where $\kappa_{0}:=\KL\left(p_{0}\Vert\mathcal{N}(0,I_{d})\right)$ and $\alpha_1$ is defined through \eqref{noise_kernel}. 
\end{corollary}

\begin{proof}
By applying the triangle inequality for the 2-Wasserstein distance, Young's inequality and Proposition~\ref{prop:diffusion:bounded},
we have
\begin{align}
\mathcal{W}_{2}^{2}(\tilde{p}_{0},p_{0})
&\leq 2\left(\mathcal{W}_{2}^{2}(\tilde{p}_{0},\hat{p}_{0})+\mathcal{W}_{2}^{2}(\hat{p}_{0},p_{0})\right)
\nonumber\\
& = 2\left(2\sqrt{2\KL\left(p_{0}\Vert\mathcal{N}(0,I_{d})\right)}R^{2}\alpha_{1} +\mathcal{W}_{2}^{2}(\hat{p}_{0},p_{0})\right).
\end{align}
The rest of the proof follows the same lines as in the proof of Corollary~\ref{cor:diffusion} and is hence omitted here.
\end{proof}

Note that in Corollary~\ref{cor:diffusion:bounded}, 
the right hand side of \eqref{cor:bounded:RHS}
can be made small if
$\alpha_{1}=\exp(-\frac{1}{2}\int_{0}^{1}\beta_{\tau}\d\tau)$
is chosen to be small.  
Indeed, we have the following corollary.

\begin{corollary}\label{cor:bounded:small}
Under the setting of Corollary~\ref{cor:diffusion:bounded},
\begin{equation}
\mathcal{W}_{2}^{2}(\tilde{p}_{0},p_{0})  \leq 2\left(2\sqrt{2}R^{2}+ e  \right)\varepsilon_1,
\end{equation}
provided that $\alpha_{1}\leq\varepsilon_{1}/\sqrt{\kappa_{0}}$.
\end{corollary}

\begin{proof}
It is an immediate consequence of Corollary~\ref{cor:diffusion:bounded}.
\end{proof}

Similarly, for Corollary~\ref{cor:diffusion}
for the distributions with unbounded support, 
the right hand side of \eqref{cor:RHS}
in Corollary~\ref{cor:diffusion} can be made small
if $\alpha_{1}=\exp(-\frac{1}{2}\int_{0}^{1}\beta_{\tau}\d\tau)$
is chosen to be small.  
Indeed, we have the following corollary.

\begin{corollary}\label{cor:small}
Under the setting of Corollary~\ref{cor:diffusion},
\begin{equation}
\mathcal{W}_{2}^{2}(\tilde{p}_{0},p_{0})  
\leq 2\left(L^{2}+e\right)\varepsilon_{1},
\end{equation}
provided that $\alpha_{1}\leq\sqrt{\frac{-m_{2}+\sqrt{m_{2}^{2}+4\varepsilon_{1}d}}{2d}}$.
\end{corollary}

\begin{proof}
First, we recall from Corollary~\ref{cor:diffusion} that
\begin{align}
\mathcal{W}_{2}^{2}(\tilde{p}_{0},p_{0})  \leq 2\left(L^2 \left(\alpha_{1}^{2}m_{2}+(1-\sigma_{1})^{2}d\right) + e\varepsilon_1 \right),
\end{align}
where $\alpha_{1}=\exp(-\frac{1}{2}\int_{0}^{1}\beta_{\tau}\d\tau)$
and $\sigma_{1}=\sqrt{1-\alpha_{1}^{2}}$.

Next, we can compute that
\begin{align}
\alpha_{1}^{2}m_{2}+(1-\sigma_{1})^{2}d
&=\alpha_{1}^{2}m_{2}+\left(1-\sqrt{1-\alpha_{1}^{2}}\right)^{2}d
\nonumber
\\ &=\alpha_{1}^{2}m_{2}+\left(\frac{\alpha_{1}^{2}}{1+\sqrt{1-\alpha_{1}^{2}}}\right)^{2}d
\nonumber
\\
&\leq
\alpha_{1}^{2}m_{2}+\alpha_{1}^{4}d
\leq\varepsilon_{1},
\end{align}
provided that
$\alpha_{1}^{2}\leq\frac{-m_{2}+\sqrt{m_{2}^{2}+4\varepsilon_{1}d}}{2d}$.
This completes the proof.
\end{proof}

\section{Inverse Problems Using STMD} \label{App_inverse_problems}

Inverse problems aim at constructing samples from the high dimensional distributions given partial observations. For example, in the context of Section~\ref{unc_stmd}, consider the problem of generating samples from the data distribution $x_0 \sim p_0$, conditional on noisy observations 
\begin{equation}
    y = \mathcal{A}(x_0) + \eta, \quad \eta \sim \N\left(0, \sigma_y^2 I_{n}\right),
\end{equation}
where $\mathcal{A}$ is a linear or nonlinear measurement operator and $y\in \R^n$ with $n<d$. There are many methods in the literature for solving such problems; see, for example, \cite{song2020score, chung_diffusion_2024, graikos_fast_2025, peng_noise_2025} and the references therein. 
Because solving inverse problems is not the primary purpose of our work, we will follow the simple setup of \cite{wang2022zero}, and leave the design of an inverse problem solver tailored to our STMD algorithm for future work.
To this end, given linear and noiseless measurement $y = M x$, where $M \in \R^{n \times d}$, and using $M^+ = M^{\top}(M M^{\top} )^{-1}$, we provide the following simple inference algorithm.

\begin{algorithm}[!ht] 
\caption{STMD conditional inference}\label{inference_algo_inpaint} 
\hspace*{\algorithmicindent}\textbf{Input:} Trained model $u^\theta(z_s, r, s, x_t, t)$, initial samples $x_1 \sim \N(0, I_d)$, number of inference steps $n_{\mathrm{inf}}$, number of Mean Flow steps $n_{\mathrm{mf}}$, mask $M$, observation $y=M x_0$. 
\begin{algorithmic}
\State $\Delta t \leftarrow \frac{1}{n_\mathrm{inf}}$, $\Delta s \leftarrow \frac{1}{n_\mathrm{mf}}$, $t \leftarrow 1$
\For{$k$ in $\mathtt{range}(n_{\mathrm{inf}})$}:
    \State $z_1 \sim \N(0, I_d)$, $s \leftarrow 1$
    \For{$i$ in $\mathtt{range}(n_{\mathrm{mf}})$}: 
        \State $z_{s - \Delta s} \leftarrow z_s - \Delta s \,  u^\theta(z_s, s - \Delta s, s, x_t, t)$
        \State $s \leftarrow s - \Delta s$
    \EndFor
    \State $x_0 \leftarrow z_0$
    \State $x_0 \leftarrow M ^{+} y + (1 - M^{+} M) x_0$
    \State $x_{t-\Delta t} \sim p(x_{t-\Delta t}\mid \bar x_0, x_t)$ from Equation \eqref{bridge_measure}
    \State $x_t \leftarrow x_{t-\Delta t}$
    \State $t \leftarrow t - \Delta t$
\EndFor
\State
\Return $ x_0$
\end{algorithmic} 
\end{algorithm}

\section{Experimental Details}\label{App_exp_details}

The MNIST and CIFAR10 experiments use U-Net architectures using the diffusers library. 
We list some key parameters of the networks below in Table~\ref{tab:network_params}.

\begin{table}[!ht]
    \centering
    \caption{Network parameters for MNIST and CIFAR10.}
    \label{tab:network_params}
    \resizebox{\textwidth}{!}{%
    \begin{tabular}{lcc}
        \toprule
        \textbf{Feature} & \textbf{MNIST} & \textbf{CIFAR10} \\
        \midrule
        Total Params 
        & 1.1 M 
        & 37.5 M \\
        
        Down Block Types 
        & \texttt{\{DownBlock2D, AttnDownBlock2D, DownBlock2D\}} 
        & \texttt{\{DownBlock2D, AttnDownBlock2D, DownBlock2D, DownBlock2D\}} \\
        
        Up Block Types 
        & \texttt{\{UpBlock2D, AttnUpBlock2D, UpBlock2D\}} 
        & \texttt{\{UpBlock2D, UpBlock2D, AttnUpBlock2D, UpBlock2D\}} \\
        
        Block Out Channels 
        & \texttt{\{32, 64, 32\}} 
        & \texttt{\{128, 256, 256, 256\}} \\

        Layers per Block 
        & \texttt{1} 
        & \texttt{2} \\
        
        Time Embedding Type 
        & \texttt{positional} 
        & \texttt{positional} \\
        
        Activation Function 
        & \texttt{SiLu} 
        & \texttt{SiLu} \\
        
        Downsample Padding 
        & \texttt{1} 
        & \texttt{0} \\
        
        Input Sample Size 
        & \texttt{1x28x28} 
        & \texttt{3x32x32} \\
        \bottomrule
    \end{tabular}%
    }
\end{table}

We train with the following parameters in each experiment summarized in Table~\ref{tab:training_hyperparams}.

\begin{table}[!ht]
    \centering
    \caption{Training hyperparameters for MNIST and CIFAR10.}
    \label{tab:training_hyperparams}
    \begin{tabular}{lcc}
        \toprule
        \textbf{Hyperparameter} & \textbf{MNIST} & \textbf{CIFAR10} \\
        \midrule
        Learning rate & 0.0005 & 0.0001 \\
        Total number of iterations & 150K & 450K \\
        Batch size & 64 & 64 \\
        EMA decay & .9995 & .9995 \\
        \bottomrule
    \end{tabular}
\end{table}

Furthermore, following \cite{geng2025mean, Geng-easy-2025}, for each loss function $\mathcal{L}= \|\Delta\|^2$ (where $\Delta$ is, for example, the term inside the form of \eqref{CMF}), we train using the adaptive loss $\mathtt{sg}(w) \cdot \mathcal{L}$ where 
\begin{equation}
    w = \frac{1}{(\| \Delta\|^2 + c)^p},
\end{equation}
for $c=0.01$ and $p=1$. 
We use a lognorm sampler for $r, s$ with ratio $75\%$ $r\neq t$. We refer to \cite{geng2025mean} and their public code repository for further details. 

Regarding the benchmark metric calculation for the MNIST experiment, as using the features of the InceptionV3 network is unreliable \citep{song_generative_2020}, we trained a custom classifier and used its final-layer latent features for calculating the Fr\'{e}chet distance. 
Our evaluation model is a lightweight convolutional neural network consisting of two convolutional layers followed by two fully connected layers. 
To calculate the FD, we extract the 128-dimensional pre-ReLU activations from the penultimate layer, providing a more stable representation.

For the CelebA example, we use a custom DiT model from the public Mean Flow repository \footnote{\url{https://github.com/Gsunshine/meanflow}}, with \texttt{depth}=$16$, \texttt{hidden\_size}=$1024$, \texttt{patch\_size}=$2$, \texttt{num\_heads}=$16$, and a total number of $308$M trainable parameters. 
We use the pretrained autoencoder from Stable Diffusion\footnote{\url{https://huggingface.co/stabilityai/sd-vae-ft-ema}}.
We train for $400$K iterations with a batch size of $16$ and a step size of $0.0001$. 
Finally, we provide additional generated images for all of our experiments.

\subsection{Additional MNIST Images}

\begin{figure}[H]
    \centering
    \begin{subfigure}{0.48\linewidth}
        \centering
        \includegraphics[width=\linewidth]{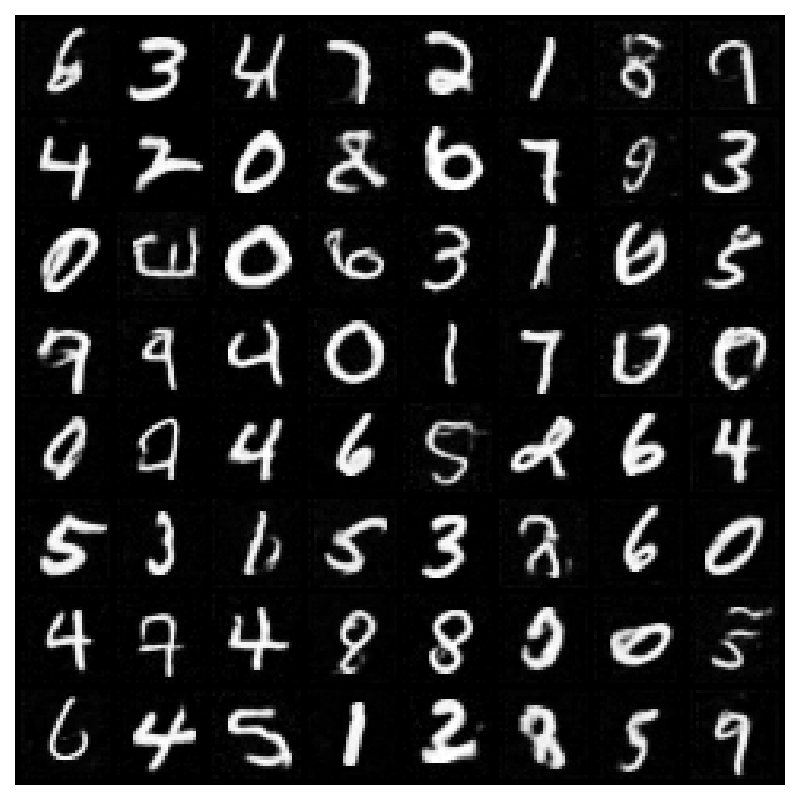}
        \caption{$n_{\mathrm{inf}}=1$, $n_{\mathrm{mf}}=1$}
        \label{fig:mnistN1n1}
    \end{subfigure}
    \hfill
    \begin{subfigure}{0.48\linewidth}
        \centering
        \includegraphics[width=\linewidth]{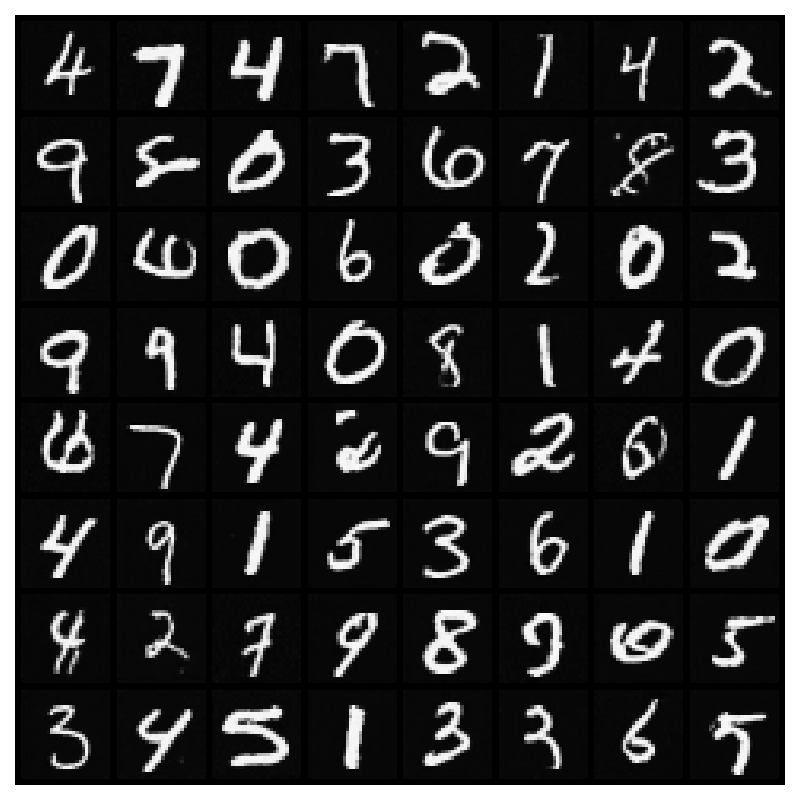}
        \caption{$n_{\mathrm{inf}}=2$, $n_{\mathrm{mf}}=2$}
        \label{fig:mnistN2n2}
    \end{subfigure}
    \caption{MNIST generation results.}
    \label{fig:mnist-generation}
\end{figure}


\subsection{Additional CIFAR10 Images}
%
\begin{figure}[H]
    \centering
    \begin{subfigure}{0.48\linewidth}
        \centering
        \includegraphics[width=\linewidth]{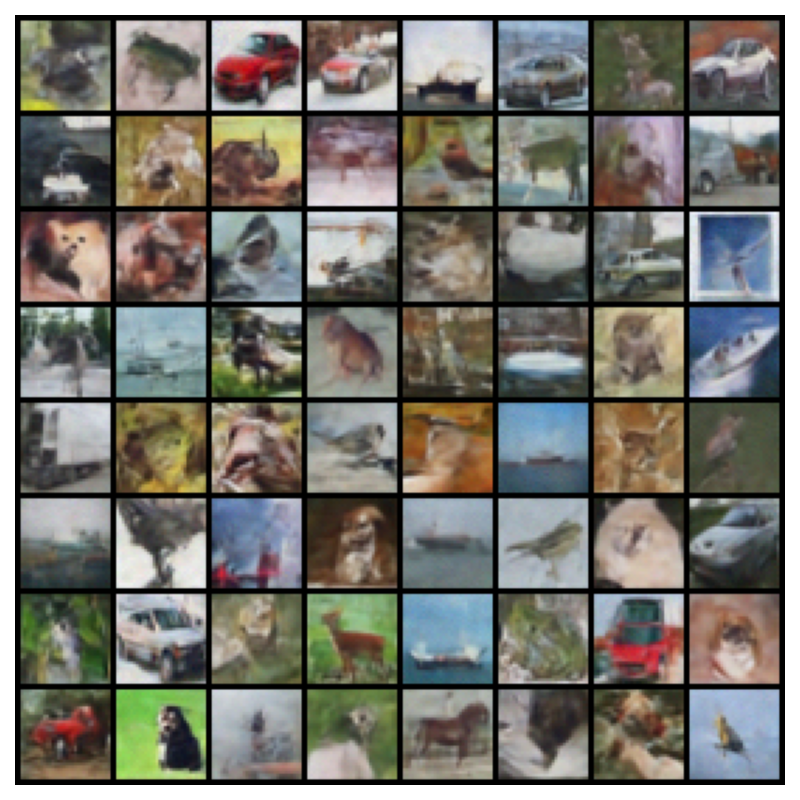}
        \caption{$n_{\mathrm{inf}}=1$, $n_{\mathrm{mf}}=1$}
        \label{fig:cifarN1n1}
    \end{subfigure}
    \hfill
    \begin{subfigure}{0.48\linewidth}
        \centering
        \includegraphics[width=\linewidth]{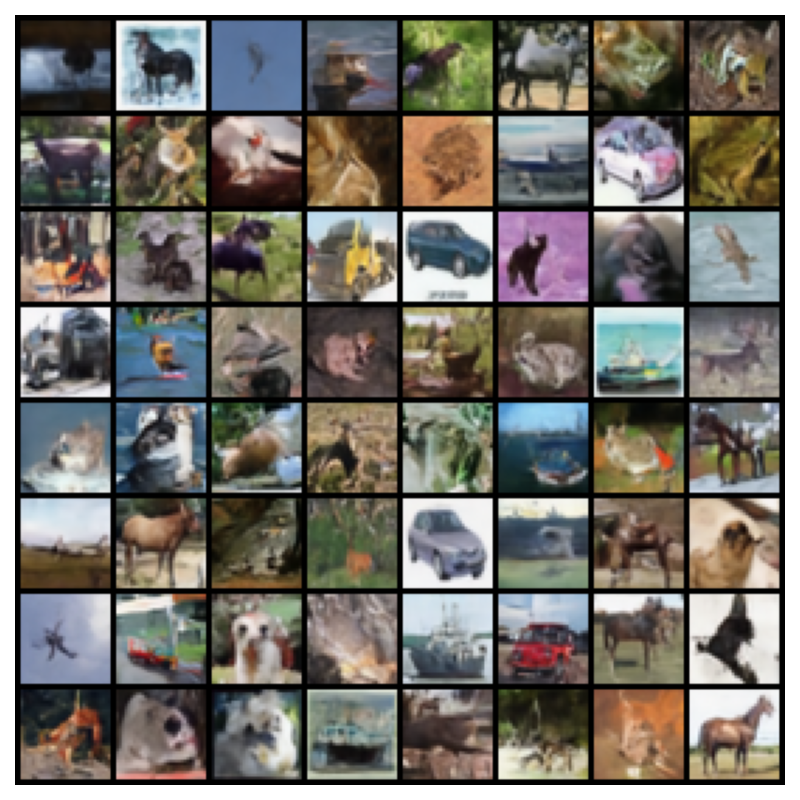}
        \caption{$n_{\mathrm{inf}}=4$, $n_{\mathrm{mf}}=2$}
        \label{fig:cifarN4n2}
    \end{subfigure}
    \caption{CIFAR10 generation results.}
    \label{fig:cifar-generation}
\end{figure}

\subsection{Additional CelebA images}
\begin{figure}[H]
    \centering
    \includegraphics[width=.8\linewidth]{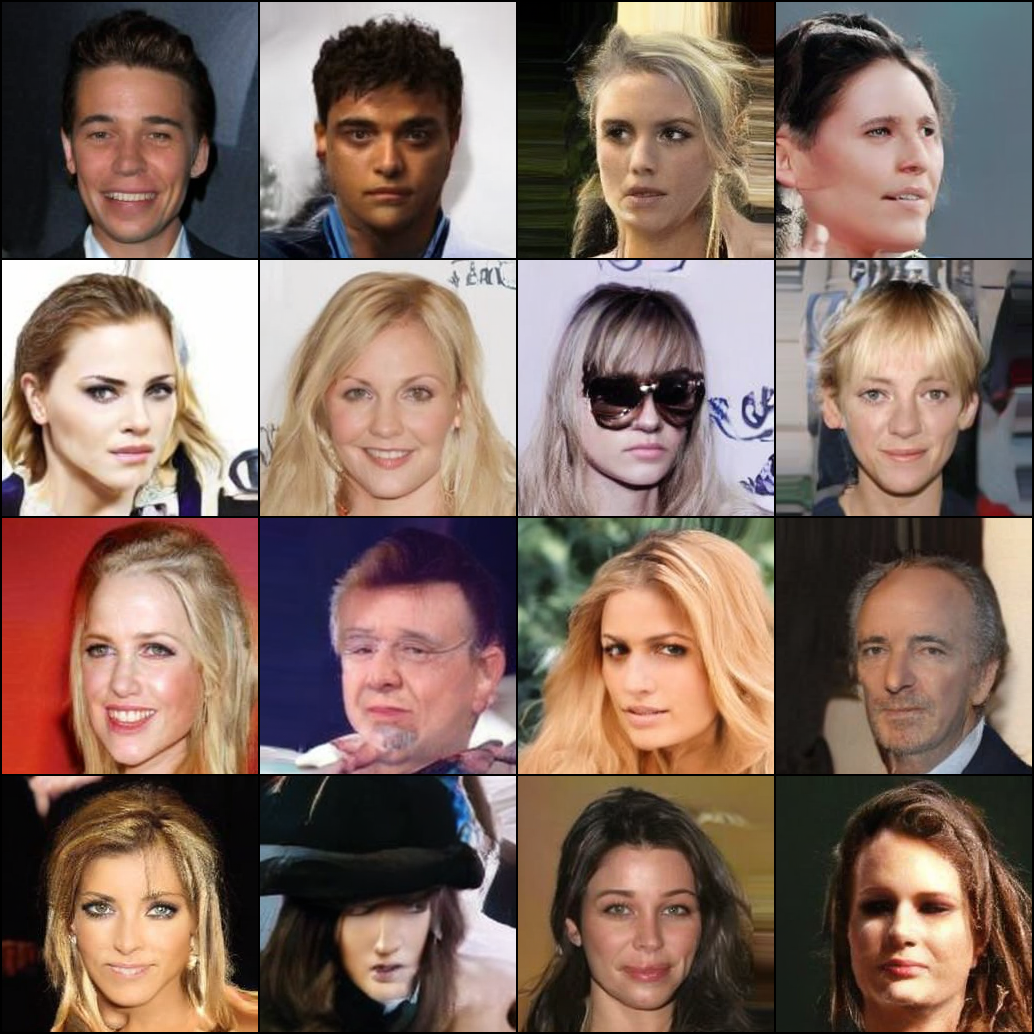}
    \caption{CelebA generation with $n_{\mathrm{inf}}=4$, $n_{\mathrm{mf}}=2$}
    \label{fig:celeba1}
\end{figure}

\end{document}